\documentclass{article}
\usepackage[margin=0.5in]{geometry}
\usepackage[labelfont=bf]{caption}
\usepackage{graphicx}
\usepackage{amssymb}
\usepackage{amsmath}
\usepackage{amsthm}
\usepackage{textcomp}
\usepackage{gensymb}
\usepackage{hyperref}
\usepackage{cleveref}
\usepackage{authblk}
\usepackage{mathtools}
\usepackage{color}
\usepackage{url}

\usepackage{bm}
\usepackage[table]{xcolor}

\usepackage[sorting=none, style=nature, url=false, isbn=false]{biblatex}
\graphicspath{{./}}

\newcommand{\bfx}{\mathbf{x}}
\newcommand{\bfs}{\mathbf{s}}
\newcommand{\bfR}{\mathbf{R}}
\newcommand{\bfxi}{\boldsymbol{\xi}}
\newcommand{\radon}{\mathcal{R}}
\newcommand{\cramer}{Cram{\'e}r }
\DeclareMathOperator*{\argmin}{arg\,min}

\newcommand{\R}{\mathbb{R}}
\newcommand{\D}{\mathbb{D}}
\newcommand{\norm}[1]{\left\lVert#1\right\rVert}
\newcommand{\abs}[1]{\left\vert#1\right\vert}
\newcommand{\E}{\mathcal{E}}
\newcommand{\Radon}{\mathcal{R}}
\newcommand{\Proj}{P}
\newcommand{\Volterra}{\mathcal{V}}
\newcommand{\ICDF}[1]{(\Volterra #1)^{-1}}

\newcommand{\dswsq}{d_{SW_2}^2}
\newcommand{\dsw}{d_{SW_2}}
\newcommand{\dhatswsq}{\hat{d}_{SW_2}^2}
\newcommand{\dhatsw}{\hat{d}_{SW_2}}

\newcommand{\dwsq}{d_{W_2}^2}

\newtheorem{theorem}{Theorem}[section]
\newtheorem{lemma}[theorem]{Lemma}
\newtheorem{proposition}[theorem]{Proposition}

\begin{document}

\title{Fast rigid alignment of heterogeneous images in sliced Wasserstein distance}







\author[1]{Yunpeng Shi}
\author[2,3]{Amit Singer}
\author[3,*]{Eric J. Verbeke}

\affil[1]{Department of Mathematics, University of California at Davis, Davis CA, USA}
\affil[2]{Department of Mathematics, Princeton University,  Princeton, NJ, USA}
\affil[3]{Program in Applied and Computational Mathematics, Princeton University, Princeton, NJ, USA}

\affil[*]{Correspondence: ev9102@princeton.edu}

\date{}

\maketitle

\begin{abstract}
\noindent Many applications of computer vision rely on the alignment of similar but non-identical images. We present a fast algorithm for aligning heterogeneous images based on optimal transport. Our approach combines the speed of fast Fourier methods with the robustness of sliced probability metrics and allows us to efficiently compute the alignment between two $L \times L$ images using the sliced 2-Wasserstein distance in $\mathcal{O}(L^2 \log L)$ operations. We show that our method is robust to translations, rotations and deformations in the images.
\end{abstract}

\section{Introduction}

Image alignment is a central task in computer vision with diverse applications across many fields. The goal of rigid image alignment is to find a linear transformation, typically a rotation and a translation, that minimizes the discrepancy of one image to another. In many applications, the images being compared may be diffeomorphic or drawn from different underlying distributions. These images are commonly encountered in areas such as medical imaging~\cite{avants_symmetric_2008}, cryogenic electron microscopy (cryo-EM)~\cite{toader_methods_2023}, and astronomy~\cite{lintott_galaxy_2008}, and we refer to these as \textit{heterogeneous} images. Rigid alignment of heterogeneous images presents an additional challenge because standard correlation methods often fail to provide meaningful alignments as they are sensitive to deformations in the images. It has been shown that the landscape induced by the Euclidean distance~\cite{singer_alignment_2024} and the standard inner product~\cite{rangan_factorization_2020} are highly irregular in the context of density and image alignment.

A fundamental component of alignment algorithms is therefore the choice of metric used to compare the images. One category of metrics, and the basis of our approach, is derived from optimal transport theory. We refer the reader to the works of~\cite{villani_optimal_2009, santambrogio_optimal_2015, peyre_computational_2020} for more background on transport based metrics. Informally, optimal transport considers the images as probability measures and the distance between them is defined as the minimal cost of transporting the mass of one image into another. Thus, transport metrics are better equipped for handling some perturbations in the images. In particular, the Earth Mover's Distance (EMD) is well known as a robust metric for image retrieval~\cite{rubner_earth_2000}. More generally, the Wasserstein distance (also Kantorovich-Rubenstein metric) has been adopted for a wide range of image processing applications~\cite{kolouri_optimal_2017, bonneel_survey_2023}. 

Despite the favorable theoretical properties of the Wasserstein distance, it becomes prohibitively expensive to compute in high dimensions between two images, let alone a large set of images that require rigid alignment. Fast methods for approximating the Wasserstein distance include entropic regularization algorithms~\cite{cuturi_sinkhorn_2013, solomon_convolutional_2015}, but even these are too demanding for most image alignment tasks. To circumvent this high computational cost, an alternative approach known as the sliced Wasserstein distance was introduced in~\cite{rabin_wasserstein_2012, bonneel_sliced_2015}. In this approach, the average Wasserstein distance from 1-D projections of the input images are taken over the circle. This method takes advantage of the fact that in 1-D, there is a closed-form solution to the Wasserstein distance, making it easy to compute (see~\cite{santambrogio_optimal_2015} for a full derivation). Since its development, the sliced Wasserstein distance has been well studied and used in many applications~\cite{Kolouri_2016_CVPR, Deshpande_2018_generative, kolouri_generalized_2019}.

In this work, we develop and investigate a fast algorithm for rigid alignment of heterogeneous images based on a minimization of the sliced 2-Wasserstein distance. Our method is shown to be 1) robust to heterogeneous images, and 2) scalable to large numbers of images. We combine transport metrics with the speed of fast Fourier methods to develop an algorithm that aligns two $L \times L$ sized images in $\mathcal{O}(L^2 \log L)$ operations. We show theoretically and experimentally that our method gives accurate rotational alignments while being robust to translations and deformations of the images. First, we formalize the problem setup in~\Cref{sec:setup}. Then we present the related works in~\Cref{sec:related} and our algorithm in~\Cref{sec:algorithm}. We then describe the theory on the stability of the sliced Wasserstein distance in~\Cref{sec:theory} and provide experimental results on the MNIST digit dataset and cryo-EM data in~\Cref{sec:results}. In~\Cref{sec:discussion}, we discuss the implications of our method and potential extensions. The theory about the robustness of our method is presented in the Appendix. Our code is freely available at \url{github.com/EricVerbeke/fast_image_alignment_ot}.


\subsection{Problem setup}
\label{sec:setup}

We consider the problem of aligning images by rigid motion. Let $F$ be a digitized $L \times L$ image representing a function $f(\bfx) : \mathbb{R}^2 \rightarrow \mathbb{R}$ which is supported on the unit disk $\mathbb{D} = \{ \bfx \in \mathbb{R}^2 : \|\bfx\| < 1\}$. Our goal is to align a target image $g(\bfx)$ to a reference image $f(\bfx)$ by finding the parameters that minimize the following:

\begin{equation}
\label{eq:rigid_alignment}
\argmin_{\bfs, \theta} d(f, T_\bfs \bfR_\theta g) ,
\end{equation}
where $d: \mathbb{D} \times \mathbb{D} \rightarrow \mathbb{R}$ denotes a generic distance, $T_\bfs$ denotes a translation by $\bfs \in \mathbb{R}^2$, and $\bfR_\theta$ denotes a rotation by $\theta \in \mbox{SO}(2)$. Specifically:

\begin{equation}
\label{eq:shift}
T_\bfs g(\bfx) = g(\bfx - \bfs) ,
\end{equation}

\begin{equation}
\label{eq:rotation}
\bfR_\theta g(\bfx) = g(\bfR_\theta^\top \bfx) .
\end{equation}
From~\cref{eq:rigid_alignment}, it is clear that the alignment parameters depend on the choice of distance used. We represent each image by a probability measure supported on the unit disk. When convenient, and under the assumption of absolute continuity with respect to Lebesgue measure, we identify this measure with its density function. This imposes that the images are positive $f(\bfx) \geq 0 \; \forall \bfx \in \mathbb{D}$ and have equal mass $\int_\mathbb{D} f(\bfx) d\bfx = \int_\mathbb{D} g(\bfx) d\bfx = 1$. The assumption of equal mass holds for images that are 2-D tomographic projections of the same 3-D object. Importantly, we allow that $f\neq g$ up to a rigid transformation. Namely, one can apply additional deformations to $f$ and $g$. 

Alignment algorithms are often categorized as local optimization or exhaustive search methods. Local optimization approaches iteratively minimize with respect to rotation and translation separately but do not guarantee the global optimum. In contrast, exhaustive methods search a dense grid of rotations and translations parameters, but are much slower. Typically, exhaustive search methods work by accelerating either the rotation or translation while brute force searching the other parameter. In our work, we develop an algorithm for fast exhaustive search over rotations. That is, we seek to find $\bfR_\theta^*$, the rotation of $g$ that minimizes the distance between $f$ and $g$. Specifically:

\begin{equation}
\label{eq:optimal_rotation}
d(f, \bfR_\theta^* g) \coloneqq \min_{\theta}  d(f, \bfR_\theta g).
\end{equation}
We note that the algorithm we develop in this work can be used in either the local optimization or exhaustive search framework for rigid alignment, however we focus our results on the case of rotational alignment for simplicity.

\subsection{Related works}
\label{sec:related}


Many diverse algorithms exist for computing fast rigid alignment of images~\cite{reddy_fft_1996, rangan_factorization_2020,  rangan_radial_2023}. In general, these approaches utilize fast transforms or compression methods to compute standard correlations between images. The main difference between these works and our approach is the choice of metric used for comparison, specifically the sliced Wasserstein distance. While the original application of the sliced Wasserstein distance to images was texture mining and computing barycenters~\cite{rabin_wasserstein_2012, bonneel_sliced_2015}, similar metrics have also been adapted for applications such as image classification~\cite{kolouri_radon_2016, shifat-e-rabbi_invariance_2023}. Further developments also explore the sliced Wasserstein distance with signed images~\cite{gong_radon_2023}. 

In particular, our work is motivated by potential applications to cryo-EM. In the context of cryo-EM, there have recently been several applications of transport based approaches. The most relevant to this work is a K-means clustering algorithm for cryo-EM images~\cite{rao_wasserstein_2020}, which utilizes an approximation of the Earth mover's distance that scales linearly in the number of images~\cite{shirdhonkar_approximate_2008}. However, the computational complexity of this approach for rotational alignment of images is $O(L^3 \log L)$ operations. See~\cite{rao_wasserstein_2020} for more detail on the wavelet EMD alignment approach. We note that the Non-Uniform Fast Fourier Transfrom (NUFFT), which we use as part of our algorithm (see~\Cref{sec:computation}) is commonly used in combination with Euclidean distance for image alignment in cryo-EM~\cite{yang_cryo-em_2008}. Additionally, NUFFTs have also been used to boost the speed of the Sinkhorn distance~\cite{lakshmanan_nonequispaced_2023}. The Wasserstein distance has already been employed for the alignment of 3-D densities~\cite{riahi_alignot_2023, riahi_empot_2023, singer_alignment_2024} and for describing the space of molecular motions~\cite{zelesko_earthmover-based_2020}. Recently, the sliced Wasserstein distance was combined with generative adversarial networks for heterogeneous reconstruction in cryo-EM~\cite{zehni_cryoswd_2023}.

\section{This work}
\label{sec:approach}

We develop a fast image alignment algorithm based on a minimization of sliced probability metrics. Briefly, for our setting, sliced probability metrics refer to distances taken between linear transforms of 1-D line projections of 2-D images, which we describe in more detail below (see also~\cite{kolouri_generalized_2022}). In particular, we are motivated by the connection of the sliced Wasserstein to the Wasserstein distance~\cite{bonnotte2013unidimensional, Kolouri_2016_CVPR, kolouri_generalized_2019, park_geometry_2023}. For two probability measures $f$ and $g$ on the bounded convex set $\Omega \in \mathbb{R}^d$, the $p$-Wasserstein distance is defined as:

\begin{equation}
\label{eq:p_wasserstein}
d_{W_p}(f, g) = \left( \underset{\gamma \in \Gamma(f, g)}{\inf} \int_{\Omega\times \Omega} \|\mathbf{x_1} - \mathbf{x_2}\|^p d\gamma(\mathbf{x_1},\mathbf{x_2}) \right)^{1/p} ,
\end{equation}
where $\Gamma(f, g)$ is the set of joint distributions on $\Omega \times \Omega$ with marginals $f$ and $g$. When $p = 1$,~\cref{eq:p_wasserstein} is commonly referred to as the Earth mover's distance. While~\cref{eq:p_wasserstein} has proven useful in many applications, the na\"ive computation between two $L \times L$ images is $\mathcal{O}(L^6)$ operations, making it intractable for high-dimensional data.

An alternative approach for handling high-dimensional data based on~\cref{eq:p_wasserstein} was first proposed by Rabin, Peyr{\'e}, Delon, and Bernot in 2012~\cite{rabin_wasserstein_2012}, and is known as the sliced Wasserstein distance. In their approach, the average Wasserstein distance from 1-D projections of each measure is taken over the circle. In the context of two images, the sliced Wasserstein distance is defined as:

\begin{equation}
\label{eq:sliced_wasserstein}
d_{SW_p}(f, g) = \left(\frac{1}{2\pi}\int_0^{2\pi} d_{W_p}^p (P_\theta f, P_\theta g) d\theta \right)^{1/p} ,
\end{equation}
where $P_\theta$ is the projection operator defined as:

\begin{equation}
\label{eq:proj_operator}
P_\theta f(x) = \int f(\bfR_\theta^\top \bfx) dy,
\end{equation}
and $\bfx = (x, y)^\top$. This approach takes advantage of the fact that in 1-D, there is a closed-form solution to~\cref{eq:p_wasserstein}, making it easy to compute (see~\Cref{sec:distances} for more detail). 

Our approach for fast, heterogeneous image alignment is to combine the scalability of the sliced Wasserstein distance with well known acceleration techniques based on fast Fourier transforms. By combining these methods, we compute the rotational alignment of two $L \times L$ images in $\mathcal{O}(L^2 \log L)$ operations using the sliced Wasserstein distance. 
We summarize our contributions as follows:

\subsection{Contributions of our work }
\begin{itemize}
    \item We develop the first $O(L^2\log L)$ OT-based image alignment algorithm, that matches the complexity of the Euclidean-based methods.
    
    \item Motivated by potential applications in cryo-EM, we theoretically prove the stability of the sliced Wasserstein distance with respect to changes in viewing angles for tomographic images. Our analysis yields a tighter bound than that presented in \cite{leeb_metrics_2023}.
    
    \item To enhance the alignment of deformable images, we improve our algorithm by introducing the ramp-filtered sliced Wasserstein distance, and demonstrate empirically that it significantly improves the alignment of MNIST digits.
\end{itemize}
We expand on the details of the algorithm in the following section.

\subsection{Algorithm}
\label{sec:algorithm}

Our algorithm consists of two main steps, 1) a linear transformation of the images and 2) fast convolution of the transformed images. First, we describe how images are converted into 1-D probability density functions (PDFs) using a Radon transform and how they can be used to compute sliced probability metrics. We use the continuous setting for simplicity, then construct the discretized algorithm used for computation in~\Cref{sec:computation}.

\subsubsection*{Radon transform}
\label{sec:transform}

As noted in~\Cref{sec:approach}, we want to compute the distances between 1-D line projections of the images. Thus, the first step of our algorithm is to obtain the 1-D line projections from each image. This can be done by taking the 2-D Radon transform, which describes the projection of an image along lines at varying angles:

\begin{equation}
\label{eq:radon_transform}
\radon f(\theta, x) = P_\theta f(x) = \int f(\bfR_\theta^\top \bfx) dy.
\end{equation}
An alternative and fast approach to obtain the 1-D line projections, which we exploit later for our algorithm (see~\Cref{sec:computation}), is to make use of the Fourier slice theorem. Throughout, we model each image as a probability measure supported on the unit disk. We assume absolute continuity with respect to Lebesgue measure, with Lipschitz continuous 
densities $f,g$. Under these hypotheses, the Radon transform $\mathcal{R}f$ is well defined as a function.

The Fourier slice theorem states that a 1-D central slice of the 2-D Fourier transform of an image is equivalent to the 1-D line projection of the image orthogonal to the direction of the slice. That is, $\mathcal{S}_\theta \mathcal{F} f = \mathcal{F} P_\theta f$, where $\mathcal{F}$ is the 2-D Fourier transform of an image, defined as:

\begin{equation}
\label{eq:2d_ft}
\mathcal{F}f(\bfx) = \widehat{f}(\bfxi) = \iint_{\mathbb{R}^2} f(\bfx) e^{-2 \pi i \langle \bfxi , \bfx \rangle} d\bfx .
\end{equation}
Here, $\langle \cdot, \cdot \rangle$ is the standard inner product and $\bfxi \in \mathbb{R}^2$ are the spatial frequencies. The slicing operator $\mathcal{S}$ denotes the restriction of a function to a lower dimension hyperplane. In our setting, we restrict the 2-D Fourier transform of an image to a central 1-D slice. Specifically:

\begin{equation}
\label{eq:slice_operator}
(\mathcal{S}_\theta \hat{f})(\xi) = \hat{f}(\xi \cos\theta, \xi \sin\theta) .
\end{equation}
The 1-D Fourier transform is defined as:

\begin{equation}
\label{eq:1d_ft}
\mathcal{F}f(x) = \widehat{f}(\xi) = \int_{\mathbb{R}} f(x) e^{-2 \pi i  \xi  x} dx .
\end{equation}
Thus, to obtain the 1-D line projections from the central slices of the 2-D Fourier transformed images, one can simply take the inverse Fourier transform of each 1-D central slice:

\begin{equation}
\label{eq:1d_ift}
\mu_\theta(x):=P_\theta f (x) = \mathcal F^{-1} \left(\widehat{f}(\bfR_\theta^\top(\xi, 0)^\top)\right) = \frac{1}{2\pi} \int  \widehat{f}(\bfR_\theta^\top(\xi, 0)^\top) e^{2 \pi i \xi x} d\xi.
\end{equation}
We denote $P_\theta f = \mu_\theta$ and $P_\theta g = \nu_\theta$ to be the 1-D line projections of images $f$ and $g$ at angle $\theta$. 

\subsubsection*{High-pass filter (optional)}
We note that $\mu_\theta$ and $\nu_\theta$ are obtained by averaging $f$ along certain directions. This averaging process may suppress too many high frequency components, which may make $\mu_\theta$ and $\nu_\theta$ less distinguishable, and therefore not suitable for certain image alignment tasks. As an alternative to directly using the line projections $\mu_\theta$ and $\nu_\theta$, one may consider their sharpened version. That is, when computing the line projections, we optionally apply a high-pass filter to each central slice before taking the inverse 1-D Fourier transform. The intuition is to downweight the low frequency components which are oversampled on the polar Fourier grid used to obtain the 1-D central slices. 

A natural high-pass filter in our case is the ramp filter, whose frequency response is proportional to the absolute frequency. Indeed, in the inverse Radon transform, the ramp filter is applied to each slice in the Fourier domain. It exactly counteracts the smoothing effect of the integration, and thereby enables exact reconstruction of the original signal from its low dimensional projections without loss of information. Specifically, we take:

\begin{equation}
\label{eq:1d_ift_ramp}
\tilde \mu_\theta(x):=h*P_\theta f (x) = \mathcal F^{-1} \left(|\xi|\widehat{f}(\bfR_\theta^\top(\xi, 0)^\top)\right) = \frac{1}{2\pi} \int |\xi| \widehat{f}(\bfR_\theta^\top(\xi, 0)^\top) e^{2 \pi i \xi x} d\xi,
\end{equation}
where $h$ is the ramp filter in the spatial domain, whose Fourier transform is $\widehat h(\xi)=|\xi|$. However, applying this filter does not enforce positivity of the line projections, which is required for probability measures. Thus, transport metrics are not directly applicable to the filtered slices. We therefore split each line projection into the sum of its positive and negative parts $\mu_\theta^+ = \max(0, \tilde\mu_\theta)$ and $\mu_\theta^- = \min(0, \tilde\mu_\theta)$, such that:

\begin{equation}
\label{eq:line_decomp}
\tilde\mu_\theta = \mu_\theta^+ + \mu_\theta^- .
\end{equation}
We always assume that $\mu_\theta^+$ and $\mu_\theta^-$ are normalized to 1. 

\subsubsection*{Sliced Wasserstein distance}
\label{sec:distances}



Using the 1-D line projections, we next define the sliced Wasserstein distance between two images and a ramp-filtered variant. For each projection angle, the Radon transform produces a one-dimensional measure 
supported on $[-1,1]$. Under absolute continuity, this measure admits a density, which we denote by $\mu_\theta$. In the algorithmic implementation we normalize 
these line measures so that they are probability measures. The cumulative distribution function (CDF) of each PDF is then:
\begin{equation}
\label{eq:cdf}
(\Volterra \mu)(t) = \int_{-\infty}^t \mu(x) dx ,
\end{equation}
where $\Volterra$ is sometimes referred to as the Volterra operator. We note that while the \cramer distance is defined between two CDFs~\cite{cramer_composition_1928} and can be used in a similar algorithm to the one we develop, we do not investigate this metric in our study. For 1-D distributions, the Wasserstein distance in~\cref{eq:p_wasserstein} has a closed form solution which can be written as the norm between the functional inverse of two CDFs. That is:

\begin{equation}
\label{eq:1d_wasserstein}
d_{W_p}(\mu, \nu) = \left( \int_{\mathbb{R}} \left| (\Volterra \mu)^{-1}(z) - (\Volterra \nu)^{-1}(z) \right|^p dz \right)^{1/p} ,
\end{equation}
where $(\Volterra \mu)^{-1}$ is defined as:

\begin{equation}
\label{eq:inv_cdf}
(\Volterra \mu)^{-1}(z) = \inf \{ t : (\Volterra \mu)(t) = z \} .
\end{equation}
The functional inverse of the CDF (ICDF) is also referred to as the quantile function. Thus we can use~\cref{eq:1d_wasserstein} in place of~\cref{eq:p_wasserstein} to compute the sliced Wasserstein distance between two images in~\cref{eq:sliced_wasserstein}. When computing the sliced Wasserstein distance without ramp filter, we use the formula:

\begin{equation}
\label{eq:sliced_wasserstein_2}
d_{SW_p}^p(f, g) = \frac{1}{2\pi}\int_0^{2\pi} d_{W_p}^p (\mu_\theta, \nu_\theta) d\theta .
\end{equation}
When the ramp filter is applied and the line projections contain both positive and negative values, we also consider the following ramp-filtered sliced Wasserstein distance (RFSW). Let:

\begin{equation}
\tilde \mu_\theta(x):= h * P_\theta f (x), \quad \mu_\theta^+=\max(0, \tilde \mu_\theta),\quad \mu_\theta^-=\min(0, \tilde \mu_\theta);
\end{equation}

\begin{equation}
\tilde \nu_\theta(x):= h * P_\theta g (x), \quad \nu_\theta^+=\max(0, \tilde \nu_\theta),\quad \nu_\theta^-=\min(0, \tilde \nu_\theta).
\end{equation}
Then:
\begin{equation}
\label{eq:signed_sliced_wasserstein}
d_{RFSW_p}^p (f, g) = \frac{1}{2\pi}\int_0^{2\pi} d_{W_p}^p (\mu_\theta^+, \nu_\theta^+) d\theta + \frac{1}{2\pi}\int_0^{2\pi} d_{W_p}^p (\mu_\theta^-, \nu_\theta^-) d\theta.
\end{equation}
A visualization of the transforms used to compute the ramp-filtered sliced Wasserstein distance is shown in~\Cref{fig:transforms}.






\begin{figure}[!ht]
	\includegraphics[width=1\textwidth]{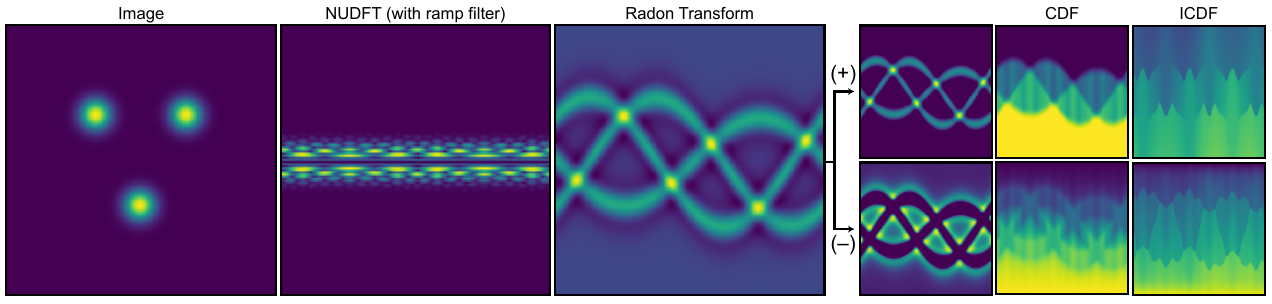}
	\centering
	\caption{Visualization of image transformations used to approximate the ramp-filtered sliced Wasserstein distance in~\cref{eq:signed_sliced_wasserstein}.}
	\label{fig:transforms}
\end{figure}





\subsection{Computation}
\label{sec:computation}

In application, $f$ and $g$ are not continuous functions, but matrices of size $L \times L$, and we therefore use the notation $F$ and $G$. 
The matrices $F$ and $G$ represent discrete cell averages of $f$ and $g$ on a uniform Cartesian grid. We are primarily interested in efficiently approximating~\cref{eq:optimal_rotation} using the sliced Wasserstein distance in~\cref{eq:sliced_wasserstein_2} and the ramp-filtered sliced Wasserstein distance in~\cref{eq:signed_sliced_wasserstein} for two images $F$ and $G$ over a fixed grid of rotations.

In the discrete implementation of the Radon transform, each projection at angle $\theta$ produces 
a finite set of values that we interpret as a discrete measure on $[-1,1]$. 
For consistency, we normalize these outputs so that they are probability 
measures. Under absolute continuity, we also identify them with their 
piecewise-constant densities on the grid.

To compute the sliced Wasserstein distance in the discrete case, we first need the discretized ICDFs from the 1-D line projections (PDFs) of each image. Let $U, V \in \mathbb{R}^{L \times n}$ be the matrices representing the Radon transforms of $f$ and $g$, whose columns are the discrete PDFs of $\mu_{\theta}$ and $\nu_{\theta}$, respectively. Here, $n$ is the number of sampled projection angles, which we choose to scale linearly with $L$ (i.e., $U$ and $V$ are typically of size $\mathbb{R}^{L \times L}$). Rather than rotating and projecting the images explicitly, we instead utilize the Fourier slice theorem by applying a NUFFT~\cite{dutt_fast_1993, barnett_parallel_2019}. Using the NUFFT, 1-D central slices of the 2-D Fourier transformed image can be obtained in $\mathcal{O}(L^2 \log L)$ operations. The discrete Radon transform of each image can then be obtained in the same complexity by applying a 1-D inverse discrete Fourier transform to the central slices of the 2-D Fourier image.

Next, we need to approximate the CDF and ICDF from each column of $U$ and $V$. We first convert the discrete PDFs to discrete CDFs by taking their respective cumulative sums. Then, the discrete ICDFs can be linearly interpolated from each discrete CDF. We remark that the computational complexity in these steps are negligible. We denote $U_C, V_C \in \mathbb{R}^{L \times n}$ to be the matrices whose columns are the discrete CDFs of the columns of $U$ and $V$. Similarly, $U_I, V_I \in \mathbb{R}^{L \times n}$ are the matrices whose columns are the functional inverse of the columns of $U_C$ and $V_C$. Then, for $p=2$, it is clear that the sliced 2-Wasserstein distance satisfies:

\begin{equation}
\label{eq:discrete_sw2}
d_{SW_2}^2(f, g) \approx \dhatsw^2(F,G):=\frac{1}{nL} \| U_I - V_I \|_F^2 ,
\end{equation}
and the ramp-filtered sliced 2-Wasserstein distance satisfies:

\begin{equation}
\label{eq:discrete_ssw2}
d_{RFSW_2}^2(f, g) \approx \hat d_{RFSW_2}^2(F,G):=\frac{1}{nL}\left( \| U_I^+ - V_I^+\|_F^2 + \| U_I^- - V_I^-\|_F^2\right),
\end{equation}
where $\| . \|_F$ denotes the Frobenius norm of a matrix, and $U_I^+ , U_I^-$ correspond to the signed ICDFs. The computational complexity for the distances in~\cref{eq:discrete_sw2} and~\cref{eq:discrete_ssw2} is therefore $\mathcal{O}(L^2 \log L)$ operations. We remark that the approximation error comes from the discretization error in Fourier transform, the linear interpolation and the sampling of projection angles. The approximation error in~\cref{eq:discrete_sw2} is controlled by the following theorem.

\begin{theorem}
\label{thm:main}
Let $f, g$ be probability measures strictly supported on the unit disk $\D$, with Lipschitz constant $K$. Assume that $f,g$ are absolutely continuous with respect to Lebesgue measure, with densities (also denoted by $f,g$ for simplicity). Let $F, G$ be their discretizations on an $L \times L$ grid using cell averages. Let $\dhatsw(F,G)$ be the sliced 2-Wasserstein distance computed using $n$ projection angles and the NUFFT-based algorithm with target precision $\epsilon$. Assume that for all $\theta$, the 1-D projected densities $\Proj_\theta f, \Proj_\theta g$ and their discrete counterparts are bounded below by a constant $\mu_{\min} > 0$ on their support. Then there exists a constant $C$, depending on $K, \mu_{\min}$, such that:
$$ \abs{\dsw^2(f,g) - \dhatsw^2(F,G)} \le C \left( \frac{1}{n} + \frac{1}{L} + \epsilon \right). $$
\\
Moreover, if the assumption $\mu_{\min}>0$ is removed, there exists a constant $C'$, depending only on $K$, such that:
\[
\bigl|\dsw^2(f,g)-\dhatsw^2(F,G)\bigr|
\;\le\;
C'\left(\frac{1}{n}\;+\;\frac{1}{\sqrt{L}}\;+\;\sqrt{\epsilon}\;\right).
\]
\end{theorem}
\noindent The proof of this theorem is left to~\Cref{app:thm_err}, where the constants $C$ and $C'$ are specified.

\paragraph{Remark.}
We note that the convergence rate in the absence of the density floor assumption $\mu_{\min}>0$ can likewise be extended to the ramp-filtered setting, which we also detail in~\Cref{app:thm_err}.

\subsubsection*{Fast computation over rotations}
Finally, to find $\bfR_\theta^*$, we need to compute the discretized distances over a fixed grid of rotations. Note that rotations of $F$ and $G$ correspond to cyclic shifts of the columns of $U$ and $V$, respectively. Thus, to approximate the sliced 2-Wasserstein distance over rotations of $G$, we compute:

\begin{equation}
\label{eq:discrete_rot_sw2}
nL\cdot\min_\theta d^2_{SW_2}(f, \bfR_\theta g) \approx \min_l \| U_I - T_l V_I\|_F^2 = \min_l \sum_{i=1}^n \| u_i - v_{i+l} \|^2 = \sum_i \| u_i \|^2 + \sum_i \| v_{i+l} \|^2 - 2 \sum_i \langle u_i , v_{i + l} \rangle ,
\end{equation}
where $T_l$ is a shift matrix. We use $u = (\Volterra \mu)^{-1}$ and $v = (\Volterra \nu)^{-1}$ to denote the discrete ICDFs from the line projections of image $F$ and $G$, respectively. Na\"ively, the distance over rotations in~\cref{eq:discrete_rot_sw2} can be computed in $\mathcal{O}(L^3)$ operations. However, it is well known that this can be accelerated using the convolution theorem~\cite{reddy_fft_1996}. Note that the first two terms in the RHS are unrelated to $l$ so we only need to compute them once. Then, we need a fast evaluation of the term $\sum_i \langle u_i , v_{i + l} \rangle$ for each $l$. We note that:

\begin{equation}
\sum_i \langle u_i , v_{i + l} \rangle=\sum_j \sum_i {U_I}_{ji} {V_I}_{j,i+l} .
\end{equation}
For each row $j$, we need to evaluate $\sum_i {U_I}_{ji} {V_I}_{j,i+l}$ for each $l$. This is exactly the cross correlation between $j$-th row of $U_I$ and $V_I$ which can be computed fast using the FFT in $O(L \log L)$ operations. For $1 \leq j \leq n$ where $n$ is linearly proportional to $L$, the total complexity for our alignment algorithm is therefore $O(L^2 \log L)$ operations. Similarly, if the ramp filter is applied, the discrete ramp-filtered sliced 2-Wasserstein distance over rotations is:

\begin{equation}
\label{eq:discrete_rot_signed_sw2}
nL\cdot d^2_{RFSW_2}(f, \bfR_\theta^* g) \approx \min_l \left( \| U_I^+ - T_l V_I^+\|_F^2 + \| U_I^- - T_l V_I^-\|_F^2 \right) .
\end{equation}
We note that many sampling schemes exist for approximating the sliced Wasserstein distance~\cite{sisouk_users_2025} and discuss them further in~\Cref{app:alignment_extended}. Timing result comparisons to other transport algorithms are shown in~\Cref{sec:timing}, and numerical results for the alignment of a rotated MNIST digit dataset are shown in~\Cref{sec:alignment}. Additionally, the alignment results are expanded in~\Cref{app:alignment_extended} to include comparisons to the Monte Carlo sliced 2-Wasserstein distance, max-sliced 2-Wasserstein distance, and Sinkhorn distance.

\section{Theory on the stability of the sliced Wasserstein distance}
\label{sec:theory}

In this section, we provide theory for some of the useful properties of the sliced Wasserstein distance. First, in~\Cref{thm:view}, we show how the sliced Wasserstein distance is stable to the change in viewing direction, which is particularly useful for the tomographic imaging setting. Then, in~\Cref{thm:shift}, we show the stability of the sliced Wasserstein distance to shifted images, which is useful when the images may be off-centered.

To begin, suppose we are given two viewing directions $a_1$ and $a_2$ where $\angle(a_1, a_2) = \theta$. Assume that the two corresponding projection images are $F$ and $G$ respectively, and let $Z$ be the volume (3-D function) defined on the compact support with radius $1$ and $\|Z\|_1=1$, WLOG.

\begin{theorem}[Stability to the change of viewing directions]
\label{thm:view}
Let $Z$ be any 3-D probability measure supported on the unit ball, and $f$ and $g$ be the projections of $Z$ whose projection directions differ by angle $\theta$. The rotation invariant sliced p-Wasserstein distance between the two projections satisfies:

\begin{equation}
\min_{\alpha}d_{SW_p}(f, \bfR_\alpha g)\leq \left(1-\frac{2p}{(p+1)\pi}\right)^{\frac{1}{p}}\theta .
\end{equation}
\end{theorem}
\noindent We remark that our result is stronger than the those of~\cite{leeb_metrics_2023} and~\cite{rao_wasserstein_2020} by a factor of $\left(1-\frac{2p}{(p+1)\pi}\right)^{1/p}$. The proof can be found in~\Cref{app:thm_view_thm_shift}.

\subsubsection*{Stability to viewing directions for the ramp-filtered case}
The ramp-filtered sliced Wasserstein distance remains stable when the viewing direction changes slightly. If two projections $I_1,I_2$ of the same density $Z$ correspond to directions separated by angle $\theta$, then:

\begin{equation}
\min_{\alpha} d_{RF\!SW_p}(I_1,\bfR_\alpha I_2)
\ \le\
\Big( 2^{p+1}\sqrt{\tfrac{8\pi}{3}}\,\tfrac{\|\nabla Z\|_{L^2}}{m_0}\,
\Big(1-\tfrac{1}{\pi}\Big)\,\theta\Big)^{\!1/p},
\end{equation}
where $m_0$ is the minimal mass of the ramp-filtered signed projection. This is a factor that only depends on $Z$, and we refer the readers to~\Cref{sec:rfsw_stability} for details. Compared to the unfiltered case (Theorem~\ref{thm:view}), the order is \emph{worse}; instead of scaling linearly in~$\theta$, the bound grows only as $\theta^{1/p}$. This degradation is inherent to the high-pass nature of the ramp filter and the nonlinear renormalization of the positive and negative parts, which force us to pass through an $L^1$ estimate before controlling the Wasserstein metric. We state and prove the precise theorem (\Cref{thm:rf_view_general_p}) in the appendix.

\begin{theorem}[Translation equivariance]
\label{thm:shift}
Assume we have the 2-D functions $f$, and its shifted version $g = T_{\bfs}f$, where $\bfs \in \mathbb R^2$ is the translation vector. Namely, $g(\bfx)=f(\bfx-\bfs)$. We assume further that $\|f\|_1=\|g\|_1=1$ and both $f$ and $g$ are supported on the unit disk. Then:

\begin{equation}
d_{SW_p}(f, g) = \left(\frac{\Gamma\left(\frac{p+1}{2}\right)}{\sqrt{\pi}\Gamma\left(\frac{p}{2}+1\right)}\right)^{\frac1p}\|\bfs\| .
\end{equation}

\end{theorem}

\subsubsection*{Translation equivariance for ramp-filtered case} 
Translation equivariance holds \emph{exactly} for $d_{RF\!SW_p}$:

\begin{equation}
d^p_{RF\!SW_p}(f, g)=\tfrac{2}{\pi}\|\mathbf s\|^p\int_0^{\pi/2}\cos^p\varphi\,d\varphi
=\tfrac{2\,\Gamma(\frac{p+1}{2})}{\sqrt{\pi}\,\Gamma(\frac{p}{2}+1)}\|\mathbf s\|^p ,
\end{equation}
when $g(\bfx)=f(\bfx-\bfs)$, since convolution, taking $(\cdot)_\pm$, and normalization all commute with 1-D shifts along each line. The precise theorem and proof can be found in~\Cref{thm:rfsw_translation}\newline

\noindent\textbf{Remark.}~
\Cref{thm:view} shows that the sliced Wasserstein distance is not very sensitive to changes in 
projection direction for tomographic images, while \Cref{thm:shift} implies that it is translation 
equivariant up to scaling. Since we could not find a suitable reference for this property in the 
literature, we provide a complete proof in \Cref{app:thm_view_thm_shift}. Neither property holds 
for the Euclidean distance.

In~\Cref{tab:view} we compare the \emph{stability bounds} of different distances with respect to 
changes in viewing direction, expressed as their order of growth in the angle difference $\theta$. 
By ``stability bound’’, we mean the rate at which the upper bounds of these 
distances decrease as $\theta\to 0$. As summarized in the table, the \emph{stability rates} of the sliced Wasserstein distance, Wasserstein distance, 
and the Euclidean distance are all linear in $\theta$. Among them, the sliced Wasserstein distance achieves a sharper constant than Wasserstein distance, 
while the Euclidean distance has the weakest constant, since it depends on the boundedness of 
$\nabla Z$ for the underlying 3-D density function. In contrast, the constants for the sliced Wasserstein distance and Wasserstein distance are absolute. 
This ordering is consistent with the numerical results in~\Cref{fig:viewing_stability}. 

The ramp-filtered sliced Wasserstein distance exhibits weaker asymptotic stability. 
For $p>1$ its rate is $\theta^{1/p}$ rather than linear, although when $p=1$ the bound is linear. 
We nevertheless use $p=2$ in practice since it allows for efficient computation. 
Interestingly, in~\Cref{fig:viewing_stability}, $d_{RF\!SW_2}$ performs slightly better than 
the Euclidean distance. This suggests that our current analysis may be conservative: 
a sharper argument might recover a linear rate, though we cannot confirm this at present. 
We note that if one applies a frequency cutoff to the ramp filter, the dependence on $\nabla Z$ 
disappears (making it theoretically stronger than the Euclidean distance in this sense), 
but the $\theta^{1/p}$ rate remains.

Although the ramp-filtered sliced Wasserstein distance is less stable to changes in viewing direction than the original sliced Wasserstein distance, 
both distances behave similarly under image translation. 
In practice, the ramp filter tends to amplify angular variation: it produces 
more distinguishable 1-D projections and hence makes it more sensitive to small in-plane 2-D rotations. 
For this reason, we expect the ramp-filtered sliced Wasserstein distance to perform better for alignment tasks as it responds strongly to small 2-D rotations while being less sensitive to other deformations such as shifts. Indeed we observe this in practice in~\Cref{sec:alignment} for alignment of MNIST digits. We emphasize that MNIST consists of general 2-D images, not tomographic data.


\begin{table}[h]
\centering
\small
\caption{Comparison of stability bounds with respect to changes in viewing direction for 
tomographic images (\Cref{thm:view} and~\Cref{sec:rfsw_stability}). We report the order of growth 
in the viewing angle difference $\theta$. For the ramp-filtered version, $C_Z$ is a constant involving the 3-D density function $Z$.}
\label{tab:view}
\begin{tabular}{lcc}
\hline
\textbf{Metric} & \textbf{Stability rate in $\theta$} & \textbf{Remarks} \\
\hline
Euclidean Distance & $\leq 2\sqrt \pi \|\nabla Z\|_\infty \theta$ & 
Proposition~2 in \cite{rao_wasserstein_2020};\\
& & linear order when $Z$ has bounded gradient
\\[6pt]
Wasserstein Distance $d_{W_p}$ & $\leq \theta$ & 
Proposition~1 in \cite{rao_wasserstein_2020}; \\
& & linear order \\[6pt]
Sliced Wasserstein Distance $d_{SW_p}$ & $\leq \bigl(1-\tfrac{2p}{(p+1)\pi}\bigr)^{1/p}\theta$ & 
\Cref{thm:view}; same order as $d_{W_p}$, \\
& & but with a sharper constant \\[6pt]
Ramp-Filtered Sliced Wasserstein Distance $d_{RF\!SW_p}$ & $\leq 2^{1+1/p}(C_Z)^{1/p}\,\theta^{1/p}$ & 
\Cref{thm:rf_view_general_p}; order $\theta^{1/p}$ \\
&& order is weaker for $p>1$,
equal for $p=1$;\\ && constant depends on $Z$ and its gradient\\[6pt]
\hline
\end{tabular}
\end{table}

\section{Numerical results}
\label{sec:results}

In this section, we provide numerical results to demonstrate the benefits of the sliced Wasserstein distance in image alignment, as well as show the performance of our algorithm. The results are organized as follows. First, in~\Cref{sec:timing}, we show timing results for computing various transport distances on images of increasing size as well as their timing for rotational alignment. In~\Cref{sec:stability}, we demonstrate the favorable properties of the sliced Wasserstein distance for image registration. Next, in~\Cref{sec:viewing}, we demonstrate that the distance in sliced Wasserstein is stable to the change in viewing angle when comparing tomographic projection images of 3-D objects. Lastly, in~\Cref{sec:alignment}, we show the experimental results of our algorithm for the alignment of a shifted and rotated MNIST digit dataset. We additionally test alignment for the MNIST dataset with additive noise. In our main results section, we focus on comparing the Euclidean distance, sliced 2-Wasserstein distance, ramp-filtered sliced 2-Wasserstein distance and 2-Wasserstein distance for simplicity; however, we include comparisons to several other metrics in~\Cref{app:alignment_extended}.

\subsection{Complexity and timing}
\label{sec:timing}

We first compare the computational complexities and timing results for computing various Wasserstein-type distances between two $L \times L$ images, which is summarized in~\Cref{tab:dense_timing}. The images are randomly generated with pixel values drawn uniformly from $[0, 1)$. The distances tested are the sliced 2-Wasserstein, ramp-filtered sliced 2-Wasserstein, convolutional Wasserstein, Sinkhorn distance and 2-Wasserstein distance. For more detail on the Sinkhorn distance and convolutional Wasserstein distance~\cite{cuturi_sinkhorn_2013, solomon_convolutional_2015}, see~\Cref{app:alignment_extended}. All timing results were carried out on a computer with a 2.6 GHz Intel Skylake processor and 256 GB of memory, and are reported as the average of $n=3$ trials. 

\begin{table}[!ht]
    \centering
    \small
    \caption{\textbf{Distance timing}. Computational complexities and timing results for calculating the distance between two random and dense $L \times L$ images using various optimal transport metrics. The Sinkhorn distance and convolutional Wasserstein distance were computed with a regularization term $\lambda = 0.01$ and $H = 3$ iterations. Time values are reported in seconds.}
    \begin{tabular}{llcccc}
        \hline
        \textbf{Metric} & \textbf{Complexity} & $32 \times 32$ & $64 \times 64$ & $96 \times 96$ & $128 \times 128$   \\ 
        \hline
        Sliced 2-Wasserstein Distance & $\mathcal{O}(L^2 \log L)$ & $ 0.0012 $ & $ 0.0025 $ & $ 0.0059 $ & $ 0.0092 $\\ 
        Ramp-Filtered Sliced 2-Wasserstein Distance & $\mathcal{O}(L^2 \log L)$ & $ 0.0016 $ & $ 0.0035 $ & $ 0.0075 $ & $ 0.0117 $\\
        Convolutional Wasserstein Distance & $\mathcal{O}(HL^2 \log L)$ & $ 0.6718 $ & $ 1.2565 $ & $ 1.312 $ & $ 1.81 $\\
        Sinkhorn Distance & $\mathcal{O}(H L^4)$ & $ 0.0378 $ & $ 0.6007 $ & $ 2.8619 $ & $ 9.0176 $\\
        2-Wasserstein Distance & $\mathcal{O}(L^6)$ & $ 0.1167 $ & $ 4.3374 $ & $ 6.5344 $ & $ 17.554 $\\ \hline
    \end{tabular}
    \label{tab:dense_timing}
\end{table}

Additionally, we report the computational complexities and timing results for computing the distances over $L$ rotations in~\Cref{tab:dense_timing_rotations}. Notably, because our formulation of the distance over rotations can be solved using FFTs, and does not require explicitly rotating the images before computing the distances (see~\Cref{sec:algorithm}), we achieve significantly faster run times compared to other transport methods. We note that the time for solving the linear program used for the 2-Wasserstein distance can vary substantially depending on the sparsity of the image. Meanwhile, the Sinkhorn distance based algorithms do not depend on the sparsity. Similarly, the convolutional Wasserstein distance only becomes advantageous over the Sinkhorn distance for larger images. These results show the advantage of our approach in terms of computational complexity and runtime.

\begin{table}[!ht]
    \centering
    \small
    \caption{\textbf{Rotational distance timing}. Computational complexities and timing results for calculating the distance between two random and dense $L \times L$ images over $L$ rotations using various optimal transport metrics. The Sinkhorn distance and convolutional Wasserstein distance were computed with a regularization term $\lambda = 0.01$ and $H = 3$ iterations. Time values are reported in seconds.}
    \begin{tabular}{llcccc}
        \hline
        \textbf{Metric} & \textbf{Complexity} & $32 \times 32$ & $64 \times 64$ & $96 \times 96$ & $128 \times 128$   \\ 
        \hline
        Sliced 2-Wasserstein Distance & $\mathcal{O}(L^2 \log L)$ & $ 0.0013 $ & $ 0.0027 $ & $ 0.0063 $ & $ 0.0097 $\\ 
        Ramp-Filtered Sliced 2-Wasserstein Distance & $\mathcal{O}(L^2 \log L)$ & $ 0.0018 $ & $ 0.0037 $ & $ 0.0081 $ & $ 0.0131 $\\
        Convolutional Wasserstein Distance & $\mathcal{O}(HL^3 \log L)$ & $ 22.0296 $ & $ 82.8461 $ & $ 127.3182 $ & $ 259.1755 $\\
        Sinkhorn Distance & $\mathcal{O}(H L^5)$ & $ 0.451 $ & $ 19.8372 $ & $ 143.2256 $ & $ 602.3988 $\\
        2-Wasserstein Distance & $\mathcal{O}(L^7)$ & $ 2.5099 $ & $ 113.381 $ & $ 576.6284 $ & $ 1985.2168 $\\ \hline
    \end{tabular}
    \label{tab:dense_timing_rotations}
\end{table}

\subsection{Stability to translation and rotation}
\label{sec:stability}

Next, to demonstrate the favorable properties of the sliced Wasserstein metrics, we construct a simplified example using Gaussian blob images. First, we demonstrate the stability of the sliced Wasserstein metrics to translations in~\Cref{fig:sw_properties}a. In this test, a reference image of size $85 \times 85$ pixels is generated and then shifted to the right by 20-pixels in 1-pixel increments. At every shift, the distance is taken to the reference under each metric. As expected from the theory in~\Cref{sec:theory}, the Wasserstein distance and sliced Wasserstein metrics increase proportionally to the magnitude of the shift, while the Euclidean distance does not. In~\Cref{fig:sw_properties}a, we plot the square root of each distance. For the ramp-filtered sliced Wasserstein distance, we use:

\begin{equation}
\label{eq:discrete_ssw_root}
d_{RFSW_2}(f, g) \approx \frac{1}{\sqrt{nL}}\left( \| U_I^+ - V_I^+\|_F^2 + \| U_I^- - V_I^-\|_F^2 \right)^{1/2} .
\end{equation}
Next, we demonstrate the stability of the sliced Wasserstein metrics to rotations in the images (see~\Cref{fig:sw_properties}b). In this test, we use the same reference image as before but compute the rotational distances of the maximally shifted image with an additional 180 degree rotation. Notably, as the Gaussian blobs no longer overlap in this experiment, the Euclidean distance does not change regardless of the rotation. However, all transport based metrics show a clear minimum at 180 degrees. These results demonstrate that the sliced Wasserstein distances provide useful metrics for heterogeneous images, while maintaining comparable behavior to the 2-Wasserstein distance.

\begin{figure}[!ht]
	\includegraphics[width=1\textwidth]{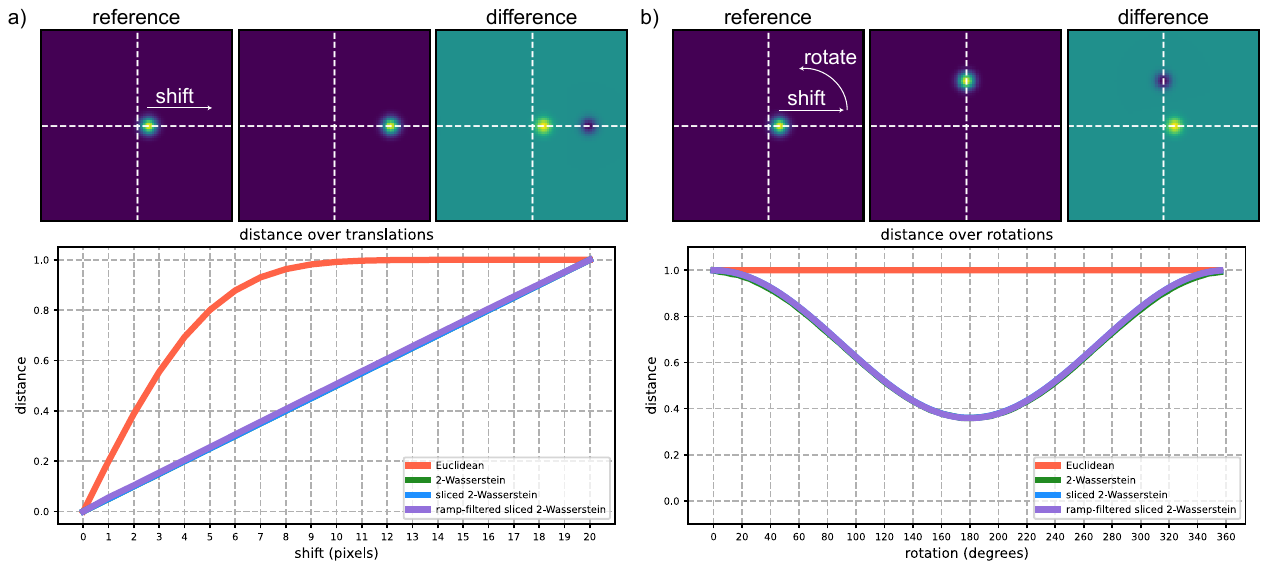}
	\centering
	\caption{Robustness of the sliced 2-Wasserstein metrics to translations and rotations in the images. Images are of size $85 \times 85$ pixels. (\textbf{a}) Stability to translation. The images show a Gaussian blob at the reference state, at the final state after being shifted by 20 pixels, and the difference between the reference and final sate.  Distances are computed at each shift. (\textbf{b}) Stability to rotation. The image to be aligned has been shifted by 20 pixels and rotated 180 degrees. The images show the reference state, after being shifted 20 pixels and rotated 90 degrees, and the difference between the two. Distances are computed over rotations. Each distance is normalized by its maximum value for scaling, and so all Wasserstein metrics appear to overlap.}
	\label{fig:sw_properties}
\end{figure}

\subsection{Stability to viewing angle}
\label{sec:viewing}

In several biomedical imaging settings, the image data is generated as a 2-D tomographic projection of a 3-D object. One prominent instance of this task arises in cryo-EM. Briefly, cryo-EM is a method used to reconstruct high-resolution 3-D structures of biological molecules from many tomographic projection images at unknown viewing angles~\cite{singer_computational_2020}. Cryo-EM 3-D reconstruction depends heavily on image alignment algorithms. Ideally, the distance between projection images with similar viewing angles would be small and increase proportionally to the difference in viewing angle. We show in~\Cref{thm:view} that similar to the results of~\cite{rao_wasserstein_2020, leeb_metrics_2023}, the sliced Wasserstein distance also satisfies this property. 

To illustrate the results on stability to viewing angle in~\Cref{tab:view}, we show how the Euclidean, sliced 2-Wasserstein, ramp-filtered sliced 2-Wasserstein and 2-Wasserstein distance change between tomographic projection images as a 3-D density is rotated out-of-plane (or equivalently, viewed from a different direction) in~\Cref{fig:viewing_stability}. The images are of size $171 \times 171$ pixels, with pixel size of $0.8117 \text{\r{A}}$/pixel, and are generated from viewing angles of $\theta \in [0, \pi/4]$. We show that while the sliced 2-Wasserstein increases with the viewing angle, the Euclidean distance plateaus rapidly. These results suggest that the sliced 2-Wasserstein distance is a more stable metric to changes in viewing angle.



\begin{figure}[!ht]
	\includegraphics[width=1\textwidth]{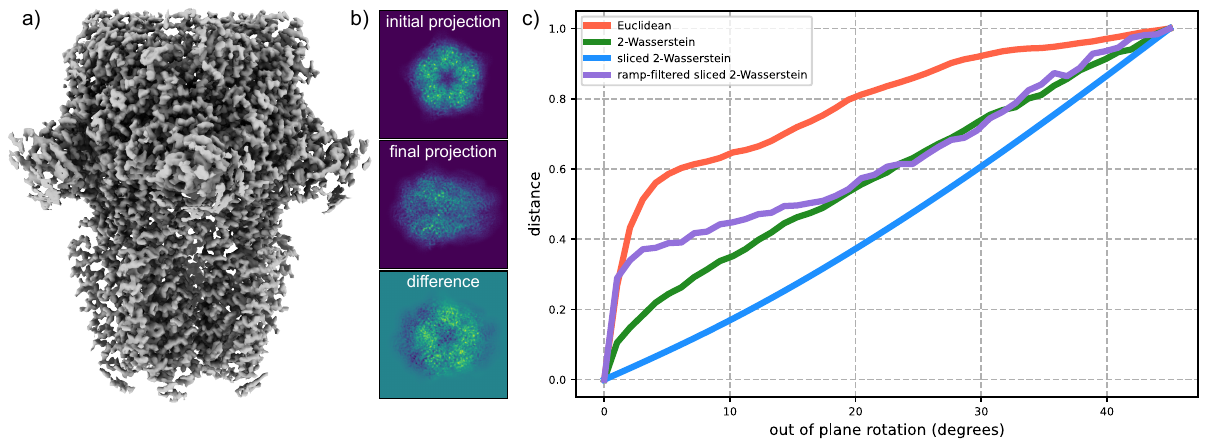}
	\centering
	\caption{Distance stability for tomographic projection images of a 3-D density (or volume) as it is rotated out-of-plane. (\textbf{a}) Cryo-EM volume (EMDB-11657). (\textbf{b}) Tomographic projection images of the cryo-EM volume at the initial viewing angle ($0$ degrees), the final viewing angle ($45$ degrees), and the difference between them. Images are size $171 \times 171$ pixels with a pixel size of $0.8117 \text{\r{A}}$/pixel. (\textbf{c}) Distances are computed between the initial projection image and projection images of the 3-D volume as it is rotated out-of-plane up to 45 degrees in 1 degree increments. Following from the theory in~\Cref{sec:theory} (summarized in~\Cref{tab:view}), these results show that the Wasserstein metrics, and in particular the sliced 2-Wasserstein distance, are stable to changes in the viewing angle between tomographic images.}
	\label{fig:viewing_stability}
\end{figure}

\subsection{Rotational alignment of heterogeneous images}
\label{sec:alignment}


We demonstrate the accuracy of our alignment algorithm for heterogeneous images using the MNIST digit data testset. This corresponds to approximately $1000$ images for each digit. We use this dataset because the inherent variability of how each handwritten digit is drawn is reflective of the heterogeneous image setting. For each non-symmetric digit $\{2, 3, 4, 5, 6, 7, 9\}$, our goal is to align the randomly rotated and translated images to a single reference of the corresponding digit. The reference for each digit is chosen to be the image that minimizes the Euclidean distance to the mean image of that digit. The alignment results should therefore favor Euclidean distance. We use the digit ``2'' for visualization because it is the digit with the largest variance.

Each image is zero-padded to size $L \times L = 39 \times 39$ pixels and rotated from a uniform distribution $\theta \sim \mathcal{U}[0, 2\pi)$. We then test the accuracy of rotational alignment at increasing translations of the images. That is, each image is shifted by $\bfs \sim \mathcal{U}(\{\pm x , \pm y \})$ for separate experiments of $(x, y) \in \{0, 2, 4, 6\}$ pixels. For rotational alignment, we set the discretization of rotation angles proportional to the image size, as described in~\Cref{sec:computation}. 

The alignments are computed using the Euclidean distance, sliced 2-Wasserstein distance, ramp-filtered sliced 2-Wasserstein distance and 2-Wasserstein distance. We omit plotting the approximations of the 2-Wasserstein distance for better visualization, but include the Sinkhorn distance in~\Cref{app:alignment_extended}. To compare the alignment results, we plot the cumulative percent of digits aligned to the reference up to $\pm 45 $ degrees in~\Cref{fig:alignment}. We use this metric for analysis since aligning non-identical images is an ill-defined problem. The exact alignment is subjective but the images are generally considered to be oriented correctly within $\pm 45$ degrees. Our results show that, in nearly all cases, the ramp-filtered sliced 2-Wasserstein provides the most accurate alignment. In particular, the transport metrics are substantially more robust when there are rotations and translations in the images.
 
\begin{figure}[!ht]
	\includegraphics[width=1\textwidth]{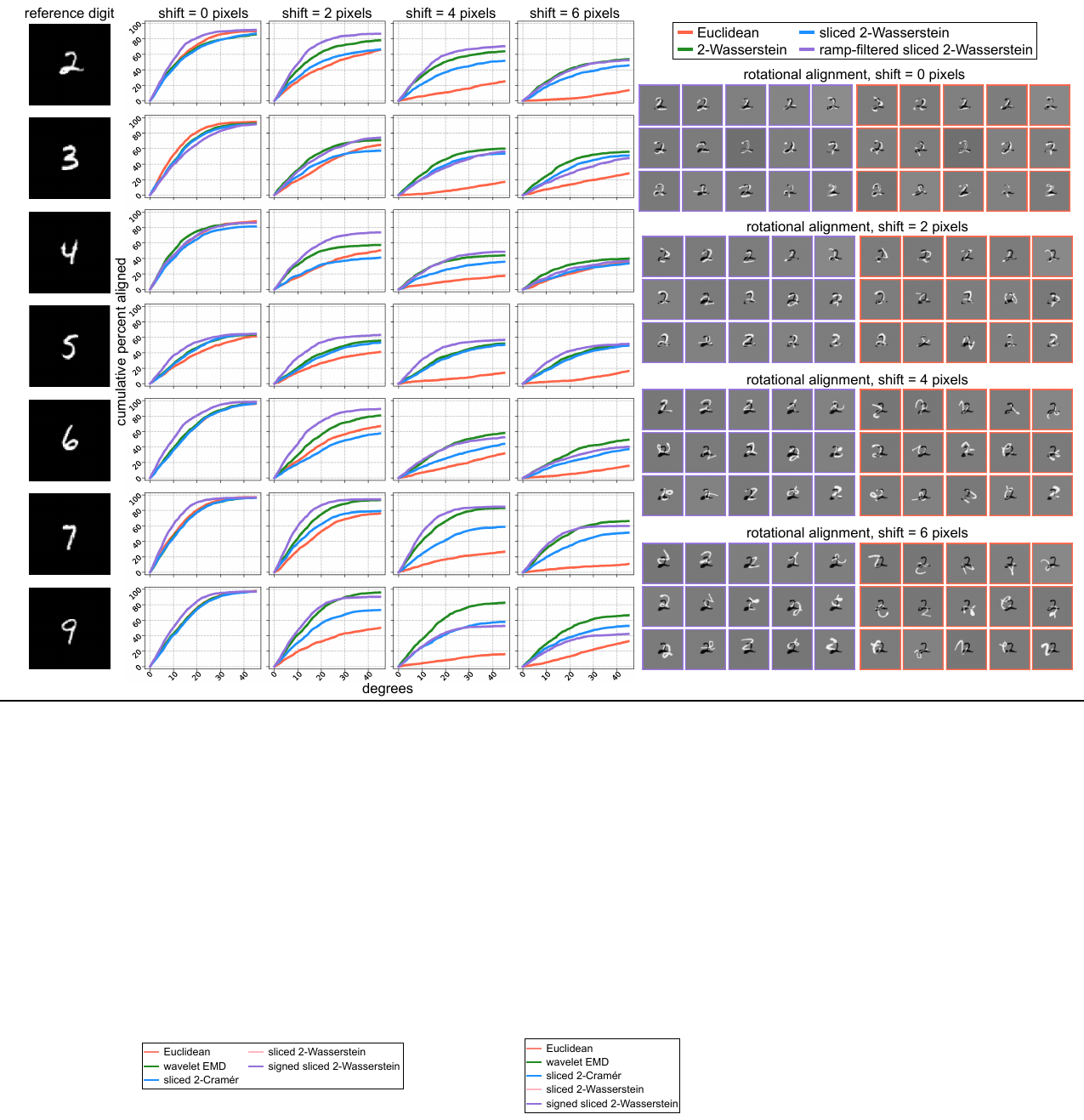}
	\centering
	\caption{Alignment of rotated and translated MNIST digits under different metrics. The reference image used for alignment of each digit is shown on the left. The line plots show the cumulative percent of digits aligned to the reference. A visualization of the first fifteen ``2'' digits aligned to the reference is shown on the right for both ramp-filtered sliced 2-Wasserstein distance and Euclidean distance at shifts of $0, 2, 4$ and $6$ pixels. The ramp-filtered sliced 2-Wasserstein distance demonstrates superior perfomance for alignment, especially in the cases where translations are also present in the images.}
	\label{fig:alignment}
\end{figure}

We report the timing results for the rotational alignment of the MNIST digit dataset in~\Cref{tab:timing}. The sliced 2-Wasserstein and ramp-filtered sliced 2-Wasserstein are $\approx 1.7 \times$ and $\approx 2.3 \times$ slower than the Euclidean distance, respectively, but still more than $20 \times$ faster than the wavelet EMD (see~\Cref{tab:all_timing}). The 2-Wasserstein distance and Sinkhorn distance were computed using the Python Optimal Transport package~\cite{flamary2021pot}. We note that in this experiment, the timing of the 2-Wasserstein outperforms the Sinkhorn distance due to the sparsity of the MNIST images. These results demonstrate that our algorithm provides fast and meaningful alignments for heterogeneous images.

\begin{table}[!ht]
    \centering
    \small
    \caption{Timing results for the rotational alignment of the MNIST dataset for the digit ``2'' using different metrics over $3$ trials. Images are size $39 \times 39$ pixels and the number of images to be aligned is $N = 1031$. The timings were carried out on a computer with a 2.6 GHz Intel Skylake processor and 32 GB of memory. (See also~\Cref{tab:all_timing} for a comparison of these results to other transport metrics).}
    \begin{tabular}{lll}
        \hline
        \textbf{Metric} & \textbf{Complexity} & \textbf{Time (seconds)} \\ 
        \hline
        Euclidean Distance & $\mathcal{O}(L^2 \log L)$ & $ 0.355 \pm 0.003 $  \\ 
        Sliced 2-Wasserstein Distance & $\mathcal{O}(L^2 \log L)$ & $ 0.586 \pm 0.006 $ \\ 
        Ramp-Filtered Sliced 2-Wasserstein Distance & $\mathcal{O}(L^2 \log L)$ & $ 0.831  \pm 0.009 $ \\ 
        2-Wasserstein Distance & $\mathcal{O}(L^7)$ & $ 1038.765 \pm 26.669 $ \\
        \hline
    \end{tabular}
    \label{tab:timing}
\end{table}

\subsubsection{Rotational alignment of noisy images}

Next, we test the effect of additive noise on rotational alignment using the MNIST digit ``2'' data. The signal-to-noise ratio (SNR) for each image is set by adding white Gaussian noise to the desired noise variance. We show alignment for images with SNR = 100, 10, 1, and 0.1, with no shift and with $\pm 3$ pixel shifts in~\Cref{fig:noise}. The reference image for alignment is the clean image described in~\Cref{sec:alignment}. Since additive noise causes the normalized images to become signed, the images are first split into positive and negative parts, where $g = g^+ + g^-$. However, since the clean image does not contain negative pixel values from noise, the alignment is computed between $f$ and $g^+$. The $g^+$ images are normalized again such that $\int_\mathbb{D} g^+(\bfx) d\bfx = 1$. Our results show that the ramp-filtered sliced 2-Wasserstein performs well at modest noise levels (SNR $>$ 1) even when shifts are present. However, the Euclidean distance outperforms all transport metrics in the case of high noise (SNR = 0.1) and no shift.

\begin{figure}[!ht]
	\includegraphics[width=1\textwidth]{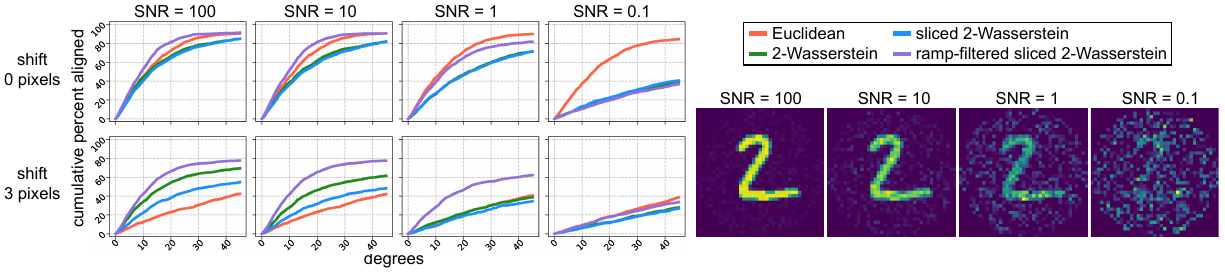}
	\centering
	\caption{Alignment of rotated, translated and noisy MNIST digits to a clean reference image under different metrics. The alignment is performed for the digit ``2'' at shifts of 0 and 3 pixels (rows) with additive white Gaussian noise at SNR = 100, 10, 1, and 0.1 (columns). Images are size $39 \times 39$ pixels and the number of images to be aligned is $N = 1031$. A representative image at each SNR is shown. These results show that the ramp-filtered sliced 2-Wasserstein distance handles moderate levels of noise, while the Euclidean distance works best at high noise and no shift.}
	\label{fig:noise}
\end{figure}



\section{Discussion}
\label{sec:discussion}

We present a fast algorithm for aligning heterogeneous images based on optimal transport and show both theoretically and experimentally that it is robust to rotation, translation and deformations in the images. As demonstration, we show alignment results for a rotated and translated MNIST digit dataset. Our approach gives increased accuracy when compared to standard correlation methods, especially when the images are not centered, without incurring much computational expense. In particular, the ramp-filtered sliced 2-Wasserstein distance exhibits the benefits of transport metrics while also being fast to compute in high-dimensions.  Additionally, our algorithm requires no hyperparameters, is easily parallelizable and can be implemented on GPUs.

Our work is motivated by the challenging task of heterogeneous 3-D reconstruction in single particle cryo-EM, where thousands to millions of high-dimensional, noisy and heterogeneous images need to be aligned~\cite{singer_computational_2020, tang_conformational_2023}. Additionally, the images are often not centered which is particularly difficult in the cryo-EM setting~\cite{heimowitz_centering_2021}. Image alignment is fundamental to heterogeneity analysis in cryo-EM, yet many methods give different results~\cite{sorzano_bias_2022}. Our results provide a novel framework for estimating image alignments that appears to be more robust against the shifts and deformations encountered in cryo-EM and other imaging modalities. However, several gaps remain before the method can be fully applied to cryo-EM data. For example, experimental images exhibit distinct amplitude contrast~\cite{contrast} and are affected by contrast transfer functions that require correction, as well as by extremely high noise levels~\cite{fastPCA} that our current approach cannot robustly handle. Potential extensions include explicitly sampling translations, investigating distance measures that are robust to both deformations and noise, and developing a three-dimensional version of the algorithm.






\clearpage
\appendix

\section{Proof of Theorem \ref{thm:main}}
\label{app:thm_err}
\subsection{Preliminaries and Notation}

We begin by formally defining the continuous and discrete quantities involved in the computation.

\subsubsection{Continuous Setting}
Let $f, g: \R^2 \to \R^+$ be two probability measures, normalized such that $\int_{\R^2} f(\bfx) d\bfx = \int_{\R^2} g(\bfx) d\bfx = 1$. 
\begin{itemize}
    \item \textbf{Support:} The functions are assumed to be strictly supported within the open unit disk, $\text{supp}(f), \, \text{supp}(g) \subset \D = \{\bfx \in \R^2 : \norm{\bfx}_2 < 1\}$. This is a common assumption for normalized images.
    \item \textbf{Lipschitz Continuity:} We assume $f$ and $g$ are Lipschitz continuous on $\D$ with Lipschitz constants $K_f$ and $K_g$, respectively. This regularity condition is essential for bounding discretization errors. Let $K=\max\{K_f, K_g\}$.
    \item \textbf{Radon Transform:} For an angle $\theta \in [0, 2\pi)$ and a scalar offset $x \in [-1, 1]$, the Radon transform $\Radon$ projects a 2-D function onto a 1-D function $\Proj_\theta f(x)$~\cite{Natterer2001}:
    $$\Proj_\theta f(x) = \int_{-\infty}^{\infty} f(\bfR_\theta^\top(x,y)^\top) dy,$$
    where $\bfR_\theta$ is the standard 2-D rotation matrix. Since $f$ is supported on $\D$, $\Proj_\theta f(x)$ is supported on $[-1, 1]$.
    \item \textbf{1-D Wasserstein Distance:} For two 1-D probability measures $\mu$ and $\nu$ on $[-1, 1]$, the squared 2-Wasserstein distance has a closed-form solution in terms of their inverse CDFs (or quantile functions)~\cite{peyre_computational_2020}:
    $$\dwsq(\mu, \nu) = \int_0^1 \abs{\ICDF{\mu}(z) - \ICDF{\nu}(z)}^2 dz,$$
    where $\ICDF{\mu}(z) = \inf\{t : (\Volterra\mu)(t) \ge z\}$ and $(\Volterra\mu)(t) = \int_{-1}^t \mu(x) dx$ is the CDF. 
    \item \textbf{Sliced 2-Wasserstein Distance:} The continuous sliced 2-Wasserstein distance between $f$ and $g$ is the mean squared 1-D Wasserstein distance over all projection angles~\cite{bonneel_sliced_2015, kolouri_generalized_2019}:
    $$\dswsq(f, g) = \frac{1}{2\pi} \int_0^{2\pi} \dwsq(\Proj_\theta f, \Proj_\theta g) d\theta.$$
\end{itemize}
Throughout, we regard 
$f,g$ as probability measures supported on the unit disk. Since they are assumed Lipschitz and hence absolutely continuous, we also work with their densities, which we denote by the same symbols for simplicity.

\subsubsection{Discrete Setting}
The continuous objects are approximated by discrete computational counterparts.
\begin{itemize}
    \item \textbf{Pixel Grid and Discrete Image:} The domain is discretized into an $L \times L$ grid of pixels, each of size $h \times h$ where $h=2/L$. Tile the square $[-1,1]^2$ by the $L\times L$ closed squares (cells) of side $h$. Define the cell-average approximation $f_L$ by
\[
f_L(x)\Big|_{C}=\frac{1}{h^2}\int_C f(y)\,dy\qquad\text{for each cell }C.
\] The matrix $F \in \R^{L \times L}$ contains these pixel values, normalized to sum to 1.
    \item \textbf{Discretization Parameters:} The analysis depends on the image resolution $L$ and the number of angular projections $n$. The projection angles are sampled uniformly: $\theta_k = 2\pi k / n$ for $k=0, \dots, n-1$.
    \item \textbf{Computed Quantities:} Let $U, V \in \R^{L \times n}$ be the matrices of computed Radon transforms, where column $U_{\cdot, k}$ is the discrete 1-D projection of $F$ at angle $\theta_k$, sampled at $L$ points. Let $U_I, V_I \in \R^{L \times n}$ be the matrices of discrete inverse CDFs, where each column is computed from the corresponding column of $U$ or $V$ via cumulative sum and linear interpolation.
    \item \textbf{Discrete Sliced Wasserstein Distance:} The final quantity computed by the algorithm is an approximation of $\dsw(f,g)$:
    $$ \dhatsw(F,G) = \sqrt{\frac{1}{n} \frac{1}{L} \norm{U_I - V_I}_F^2} = \sqrt{\frac{1}{nL} \sum_{k=0}^{n-1} \sum_{j=0}^{L-1} (U_{I,jk} - V_{I,jk})^2}. $$
    Note that the squared quantity is $\dhatswsq(F,G) = \frac{1}{nL} \norm{U_I - V_I}_F^2$. The normalization factor $1/L$ in the sum approximates the integral $\int dz$ over $[0,1]$.
\end{itemize}

\subsection{Decomposition of the Total Error}
Our objective is to bound the total approximation error $\E_{\text{total}} = \abs{\dsw(f,g) - \dhatsw(F,G)}$. A direct analysis is intractable. We therefore decompose the total error by introducing intermediate quantities and applying the triangle inequality. This standard technique allows us to isolate and bound the error from each approximation step individually. Let us define the following intermediate quantities for the \emph{squared} distances, as they are more analytically tractable:

\begin{enumerate}
    \item $d_{\text{cont}}^2 = \dswsq(f, g) = \frac{1}{2\pi} \int_0^{2\pi} \dwsq(\Proj_\theta f, \Proj_\theta g) d\theta$ (The true squared distance)
    \item $d_{\text{ang}}^2 = \frac{1}{n} \sum_{k=0}^{n-1} \dwsq(\Proj_{\theta_k} f, \Proj_{\theta_k} g)$ (Angularly discretized)
    \item $d_{\text{spat}}^2 = \frac{1}{n} \sum_{k=0}^{n-1} \dwsq(\Proj_{\theta_k} f_L, \Proj_{\theta_k} g_L)$ (Spatially discretized and NUFFT-approximated projections)
    \item $d_{\text{comp}}^2 = \dhatswsq(F,G) = \frac{1}{nL} \norm{U_I - V_I}_F^2$ (Fully discrete computed value)
\end{enumerate}
The total error on the squared distance can then be bounded as:
$$ \abs{d_{\text{cont}}^2 - d_{\text{comp}}^2} \le \underbrace{\abs{d_{\text{cont}}^2 - d_{\text{ang}}^2}}_{\E_1} + \underbrace{\abs{d_{\text{ang}}^2 - d_{\text{spat}}^2}}_{\E_2} + \underbrace{\abs{d_{\text{spat}}^2 - d_{\text{comp}}^2}}_{\E_3}. $$
This gives rise to four principal error terms to analyze:
\begin{itemize}
    \item $\E_1$: \textbf{Angular Discretization Error}, from approximating the integral over $\theta$ with a finite sum.
    \item $\E_2$: \textbf{Spatial Discretization and NUFFT Error}, from replacing continuous functions $f, g$ with their pixelated versions $f_L, g_L$ and from using the NUFFT to compute the Radon transform.
    \item $\E_3$: \textbf{1-D Quantile Function Error}, from the discrete computation of the 1-D Wasserstein distances.
\end{itemize}



\subsection{Bounding the Angular Discretization Error ($\E_1$)}
The first error source, $\E_1$, arises from approximating the continuous integral over projection angles with a discrete sum. This is a problem of numerical quadrature. The error of such an approximation is determined by the smoothness of the integrand. We will show that the integrand, $H(\theta) = \dwsq(\Proj_\theta f, \Proj_\theta g)$, is Lipschitz continuous, which guarantees an error decay rate of $O(1/n)$.

\subsubsection{Lipschitz Continuity of the Radon Transform with Respect to Angle}
The smoothness of the integrand $H(\theta)$ originates from the smoothness of the underlying functions $f$ and $g$. A change in the projection angle $\theta$ corresponds to a rotation of the function being projected.


\begin{proposition}[Lipschitz integrand for the sliced Wasserstein distance]
\label{prop:integrand_lipschitz}
Let $f,g:\R^{2}\to\R$ be probability measures
supported in $\D$. Assume that $f,g$ are absolutely continuous with respect to Lebesgue measure, 
with densities (also denoted by $f,g$ for simplicity).
Define 
\[H(\theta):=W_2^{2}(\Proj_\theta f,\Proj_\theta g).\]
Then
\[
  |H(\theta_1)-H(\theta_2)|
     \;\le\;
     8
     |\theta_1-\theta_2|
     \qquad
     \forall\theta_1,\theta_2\in[0,2\pi).
\]
\end{proposition}

\begin{proof}



\begin{align}
  &|H(\theta_1)-H(\theta_2)|
     = |W_2(u_{\theta_1}, v_{\theta_1})+W_2(u_{\theta_2}, v_{\theta_2})||W_2(u_{\theta_1}, v_{\theta_1})-W_2(u_{\theta_2}, v_{\theta_2})|
     \leq 4 |W_2(u_{\theta_1}, v_{\theta_1})-W_2(u_{\theta_2}, v_{\theta_2})|\\
     \leq & 4 (W_2(u_{\theta_1}, u_{\theta_2})+W_2(v_{\theta_1}, v_{\theta_2}))\leq 8|\theta_1-\theta_2|.
\end{align}
where the last inequality follows from Lemma \ref{lem:wtheta}.
\end{proof}

\begin{theorem}[Uniform angular quadrature error]
\label{thm:angular_error}
Let $f,g:\R^{2}\to\R$ be any probability measures
supported in $\D$. Assume that $f,g$ are absolutely continuous with respect to Lebesgue measure, 
with densities (also denoted by $f,g$ for simplicity).
Define
\(H(\theta):=W_2^{2}(u_\theta,v_\theta)\).
For the uniform grid  
\(
   \theta_k=\tfrac{2\pi k}{n},\;k=0,\dots,n-1,
\)
set  
\(
   \Delta\theta=2\pi/n.
\)
Then the discretisation error satisfies
\[
  \mathcal E_1
     :=\Bigl|
           \frac1{2\pi}\int_{0}^{2\pi}H(\theta)\,d\theta
           \;-\;
           \frac1{n}\sum_{k=0}^{n-1}H(\theta_k)
        \Bigr|
     \;\le\;\frac{8\pi\,}{n}
\]
\end{theorem}

\begin{proof}
Partition $[0,2\pi]$ into the $n$ sub-intervals
$I_k=[\theta_k,\theta_{k+1}]$ (with $\theta_n:=2\pi$).
Write the mean-value integral as a sum over these pieces:
\[
  \frac1{2\pi}\int_{0}^{2\pi}H(\theta)\,d\theta
     =\frac1{2\pi}\sum_{k=0}^{n-1}
        \int_{I_k}H(\theta)\,d\theta .
\]
On each $I_k$ we approximate the integral by the left-endpoint value
$H(\theta_k)$:
\[
  \int_{I_k}\!\!H(\theta)\,d\theta
    - H(\theta_k)\,\Delta\theta
    = \int_{I_k}\!\!\bigl(H(\theta)-H(\theta_k)\bigr)\,d\theta .
\]
Because $H$ is Lipschitz,
\(
  |H(\theta)-H(\theta_k)|
     \le 8|\theta-\theta_k|
     \le 8\Delta\theta
\)
for every $\theta\in I_k$,
so
\[
  \Bigl|
    \int_{I_k}\!\!\bigl(H(\theta)-H(\theta_k)\bigr)\,d\theta
  \Bigr|
    \le \int_{I_k}8|\theta-\theta_k|\,d\theta
    =   8\!\int_{0}^{\Delta\theta}\!\!s\,ds
    =   4(\Delta\theta)^{2}.
\]
Summing these errors and dividing by $2\pi$ gives
\[
  \mathcal E_1
     \le\frac1{2\pi}\,
          n\cdot 4\Delta\theta^{2}
     =\frac{2n}{\pi}\,
          \Bigl(\frac{2\pi}{n}\Bigr)^{\!2}
     =\frac{8\pi\,}{n},
\]
as claimed.
\end{proof}

\subsection{Bounding the Spatial Discretization and NUFFT Error ($\E_2$)}
The second error source, $\E_2$, results from replacing the continuous functions $f$ and $g$ with their pixelated approximations $f_L$ and $g_L$.

\begin{lemma}[$W_2$ Lipschitz w.r.t.\ $L^1$ under a positive density floor]
\label{lem:w2_lipschitz_positive}
Let $\mu,\nu$ be two probability \emph{densities} on the interval
$[a,b]$ such that
\(
  \mu(x),\,\nu(x)\ge\mu_{\min}>0
\)
for all $x\in[a,b]$. 
Then:
\[
  W_{2}(\mu,\nu)
     \;\le\;
     \frac{1}{\mu_{\min}}\,
     \|\mu-\nu\|_{L^{1}([a,b])},
\]
\end{lemma}

\begin{proof}
Let $\mathcal V\mu$ and $\mathcal V \nu$ be the CDFs of $\mu$ and $\nu$, and
$(\mathcal V\mu)^{-1},(\mathcal V\nu)^{-1}$ be their quantile functions.
Because $\mu\ge\mu_{\min}$, the map $\mathcal V\mu$ is strictly increasing with
$(\mathcal V\mu)'(t)\ge\mu_{\min}$, so $(\mathcal V\mu)^{-1}$ is $(1/\mu_{\min})$–Lipschitz;
the same holds for~$(\mathcal V\nu)^{-1}$.
Moreover
\(
  \|\mathcal V\mu-\mathcal V\nu\|_{L^{\infty}}
     \le\|\mu-\nu\|_{L^{1}}
\)
(by integrating the density difference).  Consequently
\[
  \|(\mathcal V\mu)^{-1} - (\mathcal V\nu)^{-1}\|_{L^{\infty}}
     \le\frac{1}{\mu_{\min}}\|\mathcal V\mu-\mathcal V\nu\|_{L^{\infty}}
     \le\frac{1}{\mu_{\min}}\|\mu-\nu\|_{L^{1}}.
\]
Since
\(
  W_{2}^{2}(\mu,\nu)
      =\int_{0}^{1}|(\mathcal V\mu)^{-1}(z) - (\mathcal V\nu)^{-1}(z)|^{2}\,dz
\)
and the integrand is bounded by the square of the $L^\infty$ norm,
the stated inequalities follow.
\end{proof}

\begin{lemma}[$W_2$ H\"older continuous w.r.t.\ $L^1$]
\label{lem:w2_holder}
Let $\mu,\nu$ be two probability densities on the interval $[a,b]$, then 
\[
  d_{W_2}(\mu,\nu) \;\le\; \frac{\sqrt2}{2}(b-a)\, \|\mu-\nu\|_{L^{1}([a,b])}^{\frac12}.
\]
\end{lemma}

\begin{proof}
By definition, it is obvious that
\[
d_{W_2}^2(\mu, \nu) \le (b-a) d_{W_1}(\mu, \nu).
\]
A fundamental identity for one-dimensional measures is that the $1$-Wasserstein distance is equal to the Cramer's distance 
\[
d_{W_1}(\mu, \nu) = \int_a^b |\mathcal{V}\mu(x) - \mathcal{V}\nu(x)| dx.
\]
Substituting this into our previous inequality, we get:

\begin{align*}
d_{W_2}^2(\mu, \nu) &\le (b-a) \int_a^b |\mathcal{V}\mu(x) - \mathcal{V}\nu(x)| dx \le (b-a) \int_a^b \left( \sup_{y \in [a,b]} |\mathcal{V}\mu(y) - \mathcal{V}\nu(y)| \right) dx \le \frac{(b-a)^2}{2} \|\mu-\nu\|_{L^{1}([a,b])}.
\end{align*}

\end{proof}
We bound the $L^1$ difference between the continuous function and its piecewise constant approximation.

\begin{theorem}[fixed-angle $L^2$ error, with NUFFT precision]\label{thm:proj_nufft}
Let $f:\mathbb{R}^2\to[0,\infty)$ be a probability measure supported in the closed unit disk
$\mathbb{D}=\{x\in\mathbb{R}^2:\,|x|\le 1\}$ and globally Lipschitz with constant $K$, i.e.
$|f(x)-f(y)|\le K\|x-y\|$ for all $x,y\in\mathbb{R}^2$. 
Fix an even integer $L=2M\ge 2$ and set $h=2/L$.

For any unit vector $\theta\in S^1$, fix a unit $\theta^\perp$ with $\theta\cdot\theta^\perp=0$, and define,
for every $g:\mathbb{R}^2\to\mathbb{R}$,
\[
(P_\theta g)(s)=\int_{\mathbf{R}} g\big(s\theta+t\,\theta^\perp\big)\,dt\qquad(s\in\mathbf{R}).
\]
If $g$ is supported in $\mathbb{D}$ then $(P_\theta g)(s)=0$ for $|s|>1$, so $P_\theta g$ is identified with an element of $L^2(-1,1)$.
Let $\widehat{h}(\omega)=\int_{\mathbf{R}} h(s)e^{-i\omega s}\,ds$ denote the one-dimensional Fourier transform and
$\widehat{H}(\xi)=\int_{\mathbf{R}^2} H(x)e^{-i\xi\cdot x}\,dx$ the two-dimensional transform.
By the Fourier slice theorem, $\widehat{P_\theta G}(\omega)=\widehat{G}(\omega\theta)$ holds for all integrable $G$ and all $\omega\in\mathbb{R}$ \cite{Natterer2001}.
Set the $L$ equispaced slice-frequencies $\omega_k=k\pi$ for $k=-M,\dots,M-1$.
Suppose a NUFFT routine returns approximations
\[
\widetilde{y}_k=\widehat{f_L}(\omega_k\theta)+\eta_k
=\widehat{P_\theta f_L}(\omega_k)+\eta_k\qquad(k=-M,\dots,M-1),
\]
where $\eta_k$ are complex errors and $\bm{\eta}=(\eta_k)_{k=-M}^{M-1}$.
Form the inverse-FFT trigonometric polynomial
\[
\widetilde{g}_L(s)=\sum_{k=-M}^{M-1}\widetilde{c}_k\,e^{ik\pi s}\qquad\text{with}\qquad
\widetilde{c}_k=\tfrac12\,\widetilde{y}_k=\tfrac12\,\widehat{P_\theta f_L}(\omega_k)+\tfrac12\,\eta_k.
\]
Then, with $g=P_\theta f$, one has the error bound
\[
\|g-\widetilde{g}_L\|_{L^2(-1,1)}\ \le\ \frac{2\sqrt{2/\pi}}{L}\,K\ +\ \frac{2\sqrt{2/\pi}}{L}\,K\ +\ \frac{1}{\sqrt{2}}\ \|\bm{\eta}\|_{\ell^2},
\]
that is,
\[
\boxed{\ \ \|P_\theta f-\widetilde{g}_L\|_{L^2(-1,1)}\ \le\ \frac{4\sqrt{2/\pi}}{L}\,K\ +\ \frac{1}{\sqrt{2}}\ \|\bm{\eta}\|_{\ell^2}\ .\ }
\]
If, moreover, the NUFFT satisfies the \emph{relative $\ell^2$ precision model}
$\|\bm{\eta}\|_{\ell^2}\le \varepsilon\Big(\sum_{k=-M}^{M-1}\big|\widehat{P_\theta f_L}(\omega_k)\big|^2\Big)^{1/2}$, then
\[
\boxed{\ \ \|P_\theta f-\widetilde{g}_L\|_{L^2(-1,1)}\ \le\ \frac{4\sqrt{2/\pi}}{L}\,K\ +\ \varepsilon\,\|P_\theta f_L\|_{L^2(-1,1)}\ .\ }
\]
\end{theorem}
\medskip
\noindent\textit{Proof.}
The inequality is proved by three contributions: (a) truncation of high 1-D Fourier modes of $P_\theta f$ beyond $|k|\ge M$; (b) spatial discretization $f\mapsto f_L$; and (c) NUFFT coefficient inaccuracy.

First, for any $f\in L^2(\mathbb{R}^2)$ supported in $\mathbb{D}$,
\[
\|P_\theta f\|_{L^2(-1,1)}\le \sqrt{2}\,\|f\|_{L^2(\mathbb{R}^2)}\qquad\text{uniformly in }\theta.
\]
To verify this, note that for each fixed $s\in[-1,1]$ the line $\{s\theta+t\theta^\perp:\,t\in\mathbb{R}\}$ intersects $\mathbb{D}$ in a single chord of length
$L(s)=2\sqrt{1-s^2}\le 2$. Thus
\[
|P_\theta f(s)|=\left|\int f(s\theta+t\theta^\perp)\,dt\right|
\le \big(L(s)\big)^{1/2}\left(\int |f(s\theta+t\theta^\perp)|^2\,dt\right)^{1/2}
\]
by Cauchy--Schwarz in the $t$-variable. Squaring and integrating $s\in[-1,1]$ gives
\[
\int_{-1}^1 |P_\theta f(s)|^2\,ds
\le \int_{-1}^1 L(s)\int |f(s\theta+t\theta^\perp)|^2\,dt\,ds
\le 2\int_{\mathbf{R}^2} |f(x)|^2\,dx,
\]
where the change of variables $x=s\theta+t\theta^\perp$ has Jacobian determinant equal to $1$.
This proves the $L^2$ operator bound.

Second, $f$ is globally Lipschitz, hence $f\in W^{1,\infty}(\mathbb{R}^2)$ with $\|\nabla f\|_{L^\infty}\le K$ and $\nabla f$ exists almost everywhere (Rademacher; see \cite{Evans2010}).
Define $G(s)=P_\theta f(s)$.
For each $s\in(-1,1)$ and any $h\to 0$,
\[
\frac{G(s+h)-G(s)}{h}
=\int_{\mathbf{R}} \frac{f\big((s+h)\theta+t\theta^\perp\big)-f\big(s\theta+t\theta^\perp\big)}{h}\,dt.
\]
The integrand is bounded in absolute value by $K$ for all $t$ by the Lipschitz property, and it vanishes whenever both arguments lie outside $\mathbb{D}$; for each fixed $s$ the relevant $t$-set has length at most $2$, so the dominating function $t\mapsto K\,\mathbf{1}_{[-2,2]}(t)$ is integrable and independent of $h$.
By dominated convergence (together with the chain rule at points where $\nabla f$ exists) the difference quotient converges in $L^1_{\mathrm{loc}}$ to
\[
G'(s)=\int_{\mathbf{R}} (\theta\cdot\nabla f)\big(s\theta+t\theta^\perp\big)\,dt
= P_\theta\big((\nabla f)\cdot\theta\big)(s)
\quad\text{for almost every }s\in(-1,1).
\]
Applying the previously proved $L^2$ bound to $u=(\nabla f)\cdot\theta$ yields
\[
\|G'\|_{L^2(-1,1)}=\|P_\theta((\nabla f)\cdot\theta)\|_{L^2(-1,1)}
\le \sqrt{2}\,\|(\nabla f)\cdot\theta\|_{L^2(\mathbf{R}^2)}
\le \sqrt{2}\,\|\nabla f\|_{L^2(\mathbf{R}^2)}\le \sqrt{2}\,K\,\sqrt{\pi},
\]
since $\nabla f=0$ a.e.\ outside $\mathbb{D}$ and $|\mathbb{D}|=\pi$.

Third, write the Fourier series of $G$ on $[-1,1]$ using the orthonormal system $\{\phi_k(s)\}_{k\in\mathbf{Z}}$ with $\phi_k(s)=e^{ik\pi s}/\sqrt{2}$:
for $c_k=(1/2)\int_{-1}^1 G(s)e^{-ik\pi s}\,ds=(1/2)\,\widehat{G}(k\pi)$ one has
$G(s)=\sum_{k\in\mathbf{Z}} c_k e^{ik\pi s}$ in $L^2(-1,1)$ and Parseval identities (see \cite[Ch.~I]{Zygmund2002})
\[
\|G\|_{L^2(-1,1)}^2=2\sum_{k\in\mathbf{Z}}|c_k|^2,
\qquad
\|G'\|_{L^2(-1,1)}^2=2\pi^2\sum_{k\in\mathbf{Z}} k^2 |c_k|^2.
\]
Let $\Pi_M$ denote the orthogonal projector onto the span of $\{e^{ik\pi s}:\,|k|\le M-1\}$.
The truncation (or tail) satisfies
\[
\|G-\Pi_M G\|_{L^2(-1,1)}^2
=2\sum_{|k|\ge M}|c_k|^2
\le \frac{2}{M^2}\sum_{k\in\mathbf{Z}} k^2 |c_k|^2
=\frac{1}{M^2\pi^2}\,\|G'\|_{L^2(-1,1)}^2.
\]
Taking square roots and inserting the bound for $\|G'\|_{L^2}$ gives
\[
\|G-\Pi_M G\|_{L^2(-1,1)}\le \frac{1}{M\pi}\,\|G'\|_{L^2(-1,1)}
\le \frac{1}{M\pi}\,\sqrt{2}\,K\sqrt{\pi}=\sqrt{2/\pi}\,\frac{K}{M}
=\frac{2\sqrt{2/\pi}}{L}\,K.
\]

Fourth, control the spatial discretization error. On each cell $C$ of side $h$, the Poincar\'e--Wirtinger inequality with sharp constant (inverse of the first nonzero Neumann eigenvalue $(\pi/h)^2$) gives
\[
\int_C |f-f_C|^2\,dx \le \frac{h^2}{\pi^2}\int_C |\nabla f|^2\,dx
\qquad\text{for }f_C=\frac{1}{h^2}\int_C f.
\]
Summing over all cells in $[-1,1]^2$,
\[
\|f-f_L\|_{L^2(\mathbf{R}^2)}^2\le \frac{h^2}{\pi^2}\int_{\mathbf{R}^2} |\nabla f|^2\,dx
\le \frac{h^2}{\pi^2}\,K^2\,|\mathbb{D}|=\frac{h^2}{\pi}\,K^2.
\]
Therefore $\|f-f_L\|_{L^2(\mathbf{R}^2)}\le hK/\sqrt{\pi}$, and by the $L^2$ bound for $P_\theta$,
\[
\|P_\theta f - P_\theta f_L\|_{L^2(-1,1)}\le \sqrt{2}\,\|f-f_L\|_{L^2(\mathbf{R}^2)}
\le \sqrt{2}\,\frac{h}{\sqrt{\pi}}\,K
=\frac{2\sqrt{2/\pi}}{L}\,K.
\]

Fifth, identify precisely what the NUFFT+IFFT returns.
Using the slice theorem \cite{Natterer2001}, $\widehat{P_\theta f_L}(\omega_k)=\widehat{f_L}(\omega_k\theta)=:y_k$.
For any $H\in L^2(-1,1)$,
\[
\int_{-1}^1 H(s)e^{-ik\pi s}\,ds=\widehat{H}(k\pi)\qquad\text{since }H=0\text{ outside }[-1,1].
\]
Hence the Fourier coefficient of $P_\theta f_L$ at frequency $k$ is exactly $c_k^{(L)}=(1/2)\widehat{P_\theta f_L}(\omega_k)=(1/2)y_k$,
and the orthogonal projector onto the span of $\{e^{ik\pi s}\}_{|k|\le M-1}$ is
\[
\Pi_M(P_\theta f_L)(s)=\sum_{k=-M}^{M-1} c_k^{(L)}\,e^{ik\pi s}=\sum_{k=-M}^{M-1} \frac12\,\widehat{P_\theta f_L}(\omega_k)\,e^{ik\pi s}.
\]
Therefore the IFFT reconstruction with inexact samples decomposes as
\[
\widetilde{g}_L(s)=\Pi_M(P_\theta f_L)(s)\;+\;\sum_{k=-M}^{M-1}\delta_k\,e^{ik\pi s},
\qquad \delta_k=\tfrac12\,\eta_k.
\]
By orthonormality of $\{e^{ik\pi s}/\sqrt{2}\}$ on $[-1,1]$ (see \cite[Ch.~I]{Zygmund2002}),
\[
\left\|\sum_{k=-M}^{M-1}\delta_k e^{ik\pi s}\right\|_{L^2(-1,1)}^2
=2\sum_{k=-M}^{M-1}|\delta_k|^2
=\frac{1}{2}\sum_{k=-M}^{M-1}|\eta_k|^2
=\frac{1}{2}\,\|\bm{\eta}\|_{\ell^2}^2,
\]
so
\[
\left\|\sum_{k=-M}^{M-1}\delta_k e^{ik\pi s}\right\|_{L^2(-1,1)}=\frac{1}{\sqrt{2}}\,\|\bm{\eta}\|_{\ell^2}.
\]

Finally, combine the three contributions using the triangle inequality and the fact that $\Pi_M$ is an $L^2$-contraction:
\[
\|P_\theta f-\widetilde{g}_L\|_{L^2(-1,1)}
\le \underbrace{\|P_\theta f-\Pi_M(P_\theta f)\|_{L^2(-1,1)}}_{\le\ \frac{2\sqrt{2/\pi}}{L}K}
+\underbrace{\|\Pi_M(P_\theta f-P_\theta f_L)\|_{L^2(-1,1)}}_{\le\ \frac{2\sqrt{2/\pi}}{L}K}
+\underbrace{\left\|\sum_{k=-M}^{M-1}\delta_k e^{ik\pi s}\right\|_{L^2(-1,1)}}_{=\ \frac{1}{\sqrt{2}}\|\bm{\eta}\|_{\ell^2}},
\]
which proves the first boxed bound.

For the second boxed bound, assume the relative $\ell^2$ model
$\|\bm{\eta}\|_{\ell^2}\le \varepsilon\left(\sum_{k=-M}^{M-1}|\widehat{P_\theta f_L}(\omega_k)|^2\right)^{1/2}$.
With $c_k^{(L)}=(1/2)\widehat{P_\theta f_L}(\omega_k)$,
\[
\sum_{k=-M}^{M-1}\big|\widehat{P_\theta f_L}(\omega_k)\big|^2
=4\sum_{k=-M}^{M-1}|c_k^{(L)}|^2
\le 4\sum_{k\in\mathbf{Z}}|c_k^{(L)}|^2
=2\,\|P_\theta f_L\|_{L^2(-1,1)}^2
\]
by Parseval. Hence $\|\bm{\eta}\|_{\ell^2}\le \varepsilon\sqrt{2}\,\|P_\theta f_L\|_{L^2(-1,1)}$, so the NUFFT term contributes at most
$(1/\sqrt{2})\cdot \varepsilon\sqrt{2}\,\|P_\theta f_L\|_{L^2(-1,1)}=\varepsilon\,\|P_\theta f_L\|_{L^2(-1,1)}$, establishing the second bound.
This completes the proof. \hfill$\square$

\bigskip
\noindent\textbf{Discussion of the NUFFT precision $\varepsilon$.}
The additive term $(1/\sqrt{2})\,\|\bm{\eta}\|_{\ell^2}$ is exact: it follows from the orthonormality of the exponential basis on $[-1,1]$ and requires no modeling assumptions beyond the definition $\eta_k=\widetilde{y}_k-\widehat{P_\theta f_L}(\omega_k)$.  
If a library guarantees \emph{pointwise} relative accuracy $|\eta_k|\le \varepsilon\,|\widehat{P_\theta f_L}(\omega_k)|$ for all $k$
(as is typical when a user tolerance \texttt{eps} is enforced via kernel width and oversampling \cite{BarnettMagland2019}), then summing yields
$\|\bm{\eta}\|_{\ell^2}\le \varepsilon\left(\sum_k |\widehat{P_\theta f_L}(\omega_k)|^2\right)^{1/2}$, hence the same
$\varepsilon\,\|P_\theta f_L\|_{L^2(-1,1)}$ contribution as in the boxed bound.
If only an \emph{absolute} $\ell^2$ guarantee $\|\bm{\eta}\|_{\ell^2}\le \varepsilon_{\mathrm{abs}}$ is available, then the additive term is $\le \varepsilon_{\mathrm{abs}}/\sqrt{2}$.

In practice there is also roundoff from the inverse FFT.  If the inverse FFT is computed in floating-point arithmetic with unit roundoff $u$, a backward-stable FFT algorithm introduces a perturbation to the coefficient vector whose $\ell^2$ norm is bounded by a factor $\gamma_{c\log L}\sim c\,u\log L$ times the coefficient norm, for a modest constant $c$ depending on implementation details \cite[Ch.~22]{Higham2002}.  This can be folded into the $\bm{\eta}$ vector (i.e., treat the effective NUFFT$+$IFFT error as $\bm{\eta}_{\mathrm{eff}}$), so the same $(1/\sqrt{2})\,\|\bm{\eta}_{\mathrm{eff}}\|_{\ell^2}$ bound applies.  Finally, note that if a different Fourier normalization (e.g.\ $2\pi$-periodic conventions) is used in software, only the harmless $1/2$ scaling of coefficients changes, while the $L$-dependence and the $\varepsilon$-dependence in the displayed inequalities remain the same.

\begin{lemma}\label{thm:norm_bound} Let $f:\mathbb{R}^2\to[0,\infty)$ 
be a probability measure supported in the unit disk 
$\mathbb{D}=\{x\in\mathbb{R}^2:\ |x|\le 1\}$. For any unit vector $\theta\in\mathbb{S}^1$ and any fixed $\theta^\perp$ with $\theta\cdot\theta^\perp=0$, define
\[
(P_\theta f)(s):=\int_{\mathbb{R}} f\big(s\theta+t\,\theta^\perp\big)\,dt,\qquad s\in\mathbb{R}.
\]
Then:
\[
\|P_\theta f\|_{L^1(\mathbb{R})}=1,
\qquad
\|P_\theta f\|_{L^2(\mathbb{R})}\le \sqrt{2}\,\|f\|_{L^2(\mathbb{R}^2)}.
\]
Moreover, since $\operatorname{supp}(P_\theta f)\subset[-1,1]$ and 
$0\le P_\theta f(s)\le 2\|f\|_{L^\infty(\mathbb{R}^2)}$, one also has
\[
\|P_\theta f\|_{L^2(\mathbb{R})}^2\le \|P_\theta f\|_{L^1(\mathbb{R})}\,\|P_\theta f\|_{L^\infty(\mathbb{R})}
\le 2\,\|f\|_{L^\infty(\mathbb{R}^2)}.
\]
If, in addition, $f$ is globally Lipschitz with constant $K$ and $\int f=1$, then 
$\|f\|_{L^\infty(\mathbb{R}^2)}\le (3K^2/\pi)^{1/3}$ (cone lower bound argument), hence
\[
\|P_\theta f\|_{L^2(\mathbb{R})}\le \sqrt{2}\,\|f\|_{L^2(\mathbb{R}^2)}
\le \sqrt{2}\,\|f\|_{L^\infty(\mathbb{R}^2)}^{1/2}
\le \sqrt{2}\,(3/\pi)^{1/6}K^{1/3}.
\]
\end{lemma}
\paragraph{Proof.}
For the $L^1$ norm,
\[
\int_{\mathbb{R}} (P_\theta f)(s)\,ds
=\int_{\mathbb{R}}\int_{\mathbb{R}} f(s\theta+t\theta^\perp)\,dt\,ds
=\int_{\mathbb{R}^2} f(x)\,dx
=1,
\]
where we used Fubini and the change of variables $x=s\theta+t\theta^\perp$ (Jacobian $=1$).

For the $L^2$ bound, fix $s\in[-1,1]$. The set of $t\in\mathbb{R}$ such that $s\theta+t\theta^\perp\in \mathbb{D}$ is a single interval (the chord of the unit disk at offset $s$) with length
$L(s)=2\sqrt{1-s^2}\le 2$. By Cauchy--Schwarz in $t$,
\[
|P_\theta f(s)|=\left|\int f(s\theta+t\theta^\perp)\,dt\right|
\le \big(L(s)\big)^{1/2}\left(\int |f(s\theta+t\theta^\perp)|^2\,dt\right)^{1/2}.
\]
Squaring and integrating in $s\in[-1,1]$ yields
\[
\int_{-1}^1 |P_\theta f(s)|^2\,ds
\le \int_{-1}^1 L(s)\int |f(s\theta+t\theta^\perp)|^2\,dt\,ds
\le 2\int_{\mathbb{R}^2} |f(x)|^2\,dx,
\]
where we again used $x=s\theta+t\theta^\perp$ with unit Jacobian. Taking square roots gives
$\|P_\theta f\|_{L^2(\mathbb{R})}\le \sqrt{2}\,\|f\|_{L^2(\mathbb{R}^2)}$.

For the stated $L^\infty$ control: since $f$ is $K$-Lipschitz, let $H=\|f\|_{L^\infty}$ and take a maximizer $x_0$ with $f(x_0)=H$. Then 
$f(x)\ge \max\{H-K\|x-x_0\|,0\}$ for all $x$. Integrating this cone lower bound over the disk $\{x:\ \|x-x_0\|\le H/K\}$ gives
\[
1=\int f \ \ge\ 2\pi\int_0^{H/K} (H-Kr)\,r\,dr
= \frac{\pi}{3}\,\frac{H^3}{K^2},
\]
hence $H\le (3K^2/\pi)^{1/3}$. Using $\|f\|_{L^2}^2\le \|f\|_{L^\infty}\|f\|_{L^1}=\|f\|_{L^\infty}$ then yields the final displayed bound.
\hfill$\square$



\begin{theorem}[Spatial discretization error]
\label{thm:spatial_error_fixed}
Assume $f,g\in L^{1}(\D)$ are both $K$-Lipschitz, and that a projection–density floor $\mu(t) \ge \mu_{\min} \geq 0$ holds for all projections. Let $f_L,g_L$ be their $L\times L$ pixelizations. Define
\[
  d_{\text{ang}}^{2}:=\frac1n\sum_{k=0}^{n-1}
     d_{W_2}^{2}\!\bigl(\Proj_{\theta_k}f,\Proj_{\theta_k}g\bigr),
  \quad
  d_{\text{spat}}^{2}:=\frac1n\sum_{k=0}^{n-1}
     d_{W_2}^{2}\!\bigl(\Proj_{\theta_k}f_L,\Proj_{\theta_k}g_L\bigr),
\]
where for $d_{\text{spat}}^{2}$ the projections are computed by NUFFT.
Then the spatial discretization error $\E_2:=|d_{\text{ang}}^{2}-d_{\text{spat}}^{2}|$ is bounded by
\[
  \E_2 \le 16\min\left\{\frac{4K}{\sqrt \pi \mu_{\min} L} + \,\frac{(3/\pi)^{\frac16}K^{\frac13
 }}{\mu_{\min}}\varepsilon\,,\, \quad \frac{2K^{\frac12}}{\pi^{\frac14} \sqrt L}\ +\,(3/\pi)^{\frac{1}{12}}K^{\frac16}\varepsilon^{\frac12}\right\}
\]
\end{theorem}

\begin{proof}
Fix an angle $\theta_k$ and abbreviate
\(u:=\Proj_{\theta_k}f,\;u_L:=\Proj_{\theta_k}f_L,\;
  v:=\Proj_{\theta_k}g,\;v_L:=\Proj_{\theta_k}g_L\).

\smallskip
First of all, because every 1-D projection lives on the interval $[-1,1]$, which has diameter 2, any Wasserstein-2 distance is bounded by $W_2(\cdot,\cdot) \le 2$. Using the identity $|a^{2}-b^{2}|=(a+b)|a-b|$ and the triangle inequality for the $W_2$ metric, we have
\begin{align}\label{eq:split}
  |d_{W_2}^{2}(u,v)-d_{W_2}^{2}(u_L,v_L)|
     &\le \bigl(d_{W_2}(u,v) + d_{W_2}(u_L,v_L)\bigr) |d_{W_2}(u,v) - d_{W_2}(u_L,v_L)| \\
     &\le 4\bigl(d_{W_2}(u,u_L)+d_{W_2}(v,v_L)\bigr). \nonumber
\end{align}

\smallskip 
By combining Lemma \ref{lem:w2_lipschitz_positive} and \ref{lem:w2_holder} we obtain that for densities supported on $[-1, 1]$,
\[
  d_{W_2}(\mu,\nu) \le \min\left\{\frac{\|\mu-\nu\|_{1}}{\mu_{\min}}, \sqrt2 \|\mu-\nu\|_{1}^{\frac12}\right\}.
\]

Applying this, followed by the $L^1$-contraction from the Proposition and the pixel-error bound from the Lemma, we get
\begin{align}
 & d_{W_2}(u,u_L) \le \min\left\{\frac{\|u-u_L\|_{1}}{\mu_{\min}},  \sqrt2\|u-u_L\|_{1}^{\frac12} \right\} \leq  \min\left\{\sqrt 2\frac{\|u-u_L\|_{2}}{\mu_{\min}},  2^{\frac34}\|u-u_L\|_{2}^{\frac12} \right\}
  \\
 \le & \min\left\{\frac{8K}{\sqrt \pi \mu_{\min} L} +\ \frac{\sqrt 2}{\mu_{\min}}\varepsilon\,\|P_\theta f_L\|_{L^2(-1,1)},\,\quad  \frac{4K^{\frac12}}{\pi^{\frac14} \sqrt L}\ +2^{\frac34}\varepsilon^{\frac12}\,\|P_\theta f_L\|^{\frac12}_{L^2(-1,1)}\right\}\\
 \le & \min\left\{\frac{8K}{\sqrt \pi \mu_{\min} L} +\ \frac{2}{\mu_{\min}}\,(3/\pi)^{1/6}K^{1/3}\varepsilon\,,\, \quad \frac{4K^{\frac12}}{\pi^{\frac14} \sqrt L}\ +\ 2\,(3/\pi)^{1/12}K^{1/6}\varepsilon^{\frac12}\right\}
\end{align}
An identical chain of inequalities holds for $d_{W_2}(v,v_L)$. 
Therefore
\begin{align}
 &|d_{W_2}^{2}(u,v)-d_{W_2}^{2}(u_L,v_L)| \le 4 (d_{W_2}(u,u_L)+d_{W_2}(v,v_L))\\
 \le &  16\min\left\{\frac{4K}{\sqrt \pi \mu_{\min} L} + \,\frac{(3/\pi)^{\frac16}K^{\frac13
 }}{\mu_{\min}}\varepsilon\,,\, \quad \frac{2K^{\frac12}}{\pi^{\frac14} \sqrt L}\ +\,(3/\pi)^{\frac{1}{12}}K^{\frac16}\varepsilon^{\frac12}\right\}
\end{align}
This provides a uniform bound for the error term of each projection. Averaging over the $n$ directions does not change the bound, so the same constant bounds the total error $\E_2$.
\end{proof}

\subsection*{Error $\boldsymbol{\mathcal{E}_3}$: Inversion by linear interpolation}
We bound $\mathcal{E}_3 = |d_{\text{spat}}^2 - d_{\text{comp}}^2|$ under the actual pipeline where, for each angle $k$, the inverse CDFs are obtained by \emph{linear interpolation} of a CDF defined on a uniform $t$–grid. Fix an angle $k$ and write $u(t),v(t)$ for the corresponding 1D projection densities on $[-1,1]$, with CDFs
\[
F_u(t):=\int_{-1}^t u(s)\,ds,\qquad F_v(t):=\int_{-1}^t v(s)\,ds.
\]
Let $t_j=-1+jh$, $j=0,\dots,L$, with $h=2/L$. Define the piecewise-linear CDFs $\widetilde F_u,\widetilde F_v$ by linearly interpolating the \emph{exact} node values $F_u(t_j)$ and $F_v(t_j)$. Let $\widetilde u,\widetilde v$ be the corresponding absolutely continuous measures (they have piecewise-constant densities equal to the cell averages of $u$ and $v$):
\[
\widetilde u'(t)=\frac{F_u(t_{j+1})-F_u(t_j)}{h}=\frac{1}{h}\int_{t_j}^{t_{j+1}}u(s)\,ds\quad(t\in[t_j,t_{j+1}]),
\]
and similarly for $\widetilde v$. Denote by $U^{-1},V^{-1}$ the (right-continuous) inverse CDFs of $u,v$ and by $\widetilde U^{-1},\widetilde V^{-1}$ those of $\widetilde u,\widetilde v$. Set
\[
A(z):=|U^{-1}(z)-V^{-1}(z)|^2,\qquad
\widetilde A(z):=|\widetilde U^{-1}(z)-\widetilde V^{-1}(z)|^2,\qquad z\in[0,1].
\]
By definition,
\[
d_{\text{spat},k}^2=\int_0^1 A(z)\,dz,\qquad
d_{\text{comp},k}^2=\frac{1}{L}\sum_{j=0}^{L-1}\widetilde A(z_j),\quad z_j:=\frac{j}{L}.
\]
Thus
\[
|d_{\text{spat},k}^2-d_{\text{comp},k}^2|
\;\le\; \underbrace{\Big|\int_0^1 A-\widetilde A\Big|}_{\text{ICDF interpolation term}}
\;+\; \underbrace{\Big|\int_0^1 \widetilde A - \tfrac{1}{L}\sum_{j=0}^{L-1}\widetilde A(z_j)\Big|}_{\text{quadrature term}}.
\]

\begin{lemma}[ICDF interpolation term]\label{lem:icdf_interp_term}
With the notation above,
\[
\Big|\int_0^1 A-\widetilde A\Big|
\;\le\; 8h \;=\; \frac{16}{L}.
\]
\end{lemma}

\begin{proof}
Use $|a^2-b^2|\le(|a|+|b|)|a-b|$ with $a=U^{-1}-V^{-1}$ and $b=\widetilde U^{-1}-\widetilde V^{-1}$, and the fact $|U^{-1}|,|V^{-1}|,|\widetilde U^{-1}|,|\widetilde V^{-1}|\le 1$:
\[
\Big|\int_0^1 A-\widetilde A\Big| \le \int_0^1 (|a|+|b|)\,|a-b|
\le 4 \int_0^1 \big(|U^{-1}-\widetilde U^{-1}| + |V^{-1}-\widetilde V^{-1}|\big).
\]
By Cauchy–Schwarz,
$\int_0^1 |U^{-1}-\widetilde U^{-1}|\le \big(\int_0^1 |U^{-1}-\widetilde U^{-1}|^2\big)^{1/2}=W_2(u,\widetilde u)$, and similarly for $v$.
It remains to show $d_{W_2}(u,\widetilde u)\le h$ (and the same for $v$). Since $F_u$ and $\widetilde F_u$ agree at all nodes $t_j$, $u$ and $\widetilde u$ assign the \emph{same mass} to each cell $I_j=[t_j,t_{j+1}]$. Couple the two measures so that all mass in $I_j$ is transported within $I_j$; then $|x-y|\le h$ on that event. Hence
\[
d_{W_2}^2(u,\widetilde u)\ \le\ \sum_{j} \big(\text{mass in }I_j\big)\cdot h^2\ =\ h^2\sum_j \text{mass in }I_j\ =\ h^2.
\]
Therefore $d_{W_2}(u,\widetilde u)\le h$ and likewise $d_{W_2}(v,\widetilde v)\le h$, giving
$\big|\int_0^1 A-\widetilde A\big|\le 4(h+h)=8h=16/L$.
\end{proof}

\begin{lemma}[Quadrature term (BV bound)]\label{lem:BV_quad}
If $B\in BV([0,1])$, then for the left Riemann sum at $z_j=j/L$,
\[
\left|\int_0^1 B(z)\,dz - \frac{1}{L}\sum_{j=0}^{L-1}B(z_j)\right|
\;\le\; \frac{\operatorname{TV}(B)}{L}.
\]
\end{lemma}

\begin{proof}
Let $I_j=[z_j,z_{j+1}]$. For $z\in I_j$, $|B(z)-B(z_j)|\le \operatorname{osc}_{I_j}(B)$, so
\[
\left|\int_{I_j}B - B(z_j)\tfrac{1}{L}\right|
=\left|\int_{I_j}(B(z)-B(z_j))\,dz\right|
\le \tfrac{1}{L}\operatorname{osc}_{I_j}(B).
\]
Sum over $j$ and use $\sum_j \operatorname{osc}_{I_j}(B)\le \operatorname{TV}(B)$.
\end{proof}

\begin{lemma}[Bounding $\operatorname{TV}(\widetilde A)$]\label{lem:TV_Atilde}
Let $\widetilde D(z):=\widetilde U^{-1}(z)-\widetilde V^{-1}(z)$. Then
$\|\widetilde D\|_{L^\infty}\le 2$ and
$\operatorname{TV}(\widetilde D)\le \operatorname{TV}(\widetilde U^{-1})+\operatorname{TV}(\widetilde V^{-1})\le 2+2=4$.
Therefore
\[
\operatorname{TV}(\widetilde A)=\operatorname{TV}(\widetilde D^2)
\ \le\ 2\,\|\widetilde D\|_{L^\infty}\,\operatorname{TV}(\widetilde D)\ \le\ 16.
\]
\end{lemma}

\begin{proof}
Each inverse CDF is nondecreasing with range in $[-1,1]$, hence total variation $=1-(-1)=2$ and sup-norm $\le 1$; the rest follows from
$\operatorname{TV}(g^2)\le 2\|g\|_\infty \operatorname{TV}(g)$.
\end{proof}

\begin{theorem}[Discrete integration error with linear‑interpolated ICDFs]\label{thm:E3_final}
For each angle $k$,
\[
|d_{\text{spat},k}^2 - d_{\text{comp},k}^2|
\;\le\; \frac{16}{L} \;+\; \frac{16}{L}
\;=\; \frac{32}{L}.
\]
Consequently,
\[
\mathcal{E}_3
= \left|\frac{1}{n}\sum_{k=0}^{n-1} d_{\text{spat},k}^2
      - \frac{1}{n}\sum_{k=0}^{n-1} d_{\text{comp},k}^2\right|
\;\le\; \frac{32}{L}.
\]
\end{theorem}

\begin{proof}
Combine Lemma~\ref{lem:icdf_interp_term} with Lemmas~\ref{lem:BV_quad} and \ref{lem:TV_Atilde}:
\[
\Big|\int_0^1 A-\widetilde A\Big|\le \frac{16}{L},\qquad
\Big|\int_0^1 \widetilde A - \tfrac{1}{L}\sum_j \widetilde A(z_j)\Big|
\le \frac{\operatorname{TV}(\widetilde A)}{L}\le \frac{16}{L}.
\]
\end{proof}

\subsection{Total Error Bound}

Having analyzed each source of error individually, we now combine them to provide a comprehensive bound on the total error between the true squared sliced Wasserstein distance, $\dswsq(f,g)$, and the final value computed by the algorithm, $\dhatswsq(F,G)$.

The total error is bounded by the sum of the four principal error components:
\begin{align*}
    \bigl|\dswsq(f,g) - \dhatswsq(F,G)\bigr| \le \E_1 + \E_2 + \E_3 
\end{align*}
Combining these gives the final bound:
\[
    \bigl|\dswsq(f,g) - \dhatswsq(F,G)\bigr| \le \frac{8\pi}{n} + 16\min\left\{\frac{4K}{\sqrt \pi \mu_{\min} L} + \,\frac{(3/\pi)^{\frac16}K^{\frac13
 }}{\mu_{\min}}\varepsilon\,,\, \quad \frac{2K^{\frac12}}{\pi^{\frac14} \sqrt L}\ +\,(3/\pi)^{\frac{1}{12}}K^{\frac16}\varepsilon^{\frac12}\right\} + \frac{32}{L}
\]
where $K=\max\{K_g\,, K_f\}$.

\paragraph{Remark (extend to the ramp filtered case).}
In the ramp–filtered pipeline the central slice is multiplied by $|\xi|$, i.e. we replace $P_\theta f$ by $\widetilde\mu_\theta:= h * P_\theta f$ with $\widehat{\widetilde\mu_\theta}(\xi)=|\xi|\,\widehat{P_\theta f}(\xi)$. By Plancherel this is just differentiation in $s$:
\[
\|\widetilde\mu_\theta\|_{L^2}=\|\partial_s P_\theta f\|_{L^2}\le \sqrt{2}\,\|(\nabla f)\!\cdot\!\theta\|_{L^2}\le \sqrt{2\pi}\,K.
\]
Consequently, in the fixed–angle decomposition
\[
\|P_\theta f-\widetilde{g}_L\|_{L^2}\;\le\;
\underbrace{\|P_\theta f-\Pi_M P_\theta f\|_{L^2}}_{\text{bandlimit}}
+\underbrace{\|\Pi_M(P_\theta f-P_\theta f_L)\|_{L^2}}_{\text{spatial}}
+\underbrace{\Big\|\sum_{k=-M}^{M-1}\delta_k e^{ik\pi s}\Big\|_{L^2}}_{\text{NUFFT}},
\]
only the \emph{constants} change when we apply the ramp filter; the \emph{orders} in $L$ and $\epsilon$ do not:
\begin{itemize}
  \item \textbf{Bandlimit (truncation).} After filtering we compare $\partial_s P_\theta f$ to its degree–$M$ trigonometric projection. The same $O(L^{-1})$ rate holds; the norm in the constant is now $\|\partial_s P_\theta f\|_{L^2}$, which is uniformly controlled by $K$ as above.
  \item \textbf{Spatial discretization.} We compare the \emph{filtered} continuous and discrete slices. The same $O(L^{-1})$ (or $O(L^{-1/2})$ without a density floor) rate carries over; the constant again involves $\|\partial_s P_\theta f\|_{L^2}\lesssim K$ instead of $\|P_\theta f\|_{L^2}$.
  \item \textbf{NUFFT term.} The frequency samples are multiplied by $|\omega|$, so the sample–wise perturbation carries an extra $|\omega|$ weight. With a relative tolerance on the (ramped) samples this only changes the \emph{constant} in the $L^2$ reconstruction error; the \emph{order} in $\epsilon$ is the same as before (and in the no–floor regime it remains $O(\sqrt{\epsilon})$ after the $L^1\!\to\!d_{W_2}$ step).
\end{itemize}
Importantly, the other two global pieces of the analysis are \emph{unchanged}: 
\(\mathcal E_1\) (angular quadrature) remains $O(1/n)$ because the Lipschitz-in-angle control is taken on the underlying projections before filtering, and \(\mathcal E_3\) (ICDF interpolation and Riemann–sum) remains $O(1/L)$ since it operates on normalized 1-D probabilities (the positive/negative parts after splitting), for which the same BV/quadrature bounds apply. In short, the ramp filter sharpens slices but only modifies constants via $\|\partial_s P_\theta f\|_{L^2}\lesssim K$; all rates in $n$, $L$, and $\epsilon$ stay the same.

\section{Proof of Theorem ~\ref{thm:view} and \ref{thm:shift}}
\label{app:thm_view_thm_shift}
\subsection{Proof for the stability to the change of viewing directions}
The two projection planes are within two great circles that differ by angle $\theta$. Assume that $l$ is the 3-D rotation axis (common line of the two projection planes), and let $l_1(\alpha)$ and $l_2(\alpha)$ be the line in the two projection images respectively, that has angle $\alpha$ with $l$. For both projections it is sufficient to consider $0\leq \alpha\leq \pi/2$ due to the symmetry. Note that in this setup, $l_1(0), l_2(0)$ are both the common line, and $l_1(\pi/2), l_2(\pi/2)$ are the lines within the two image planes that are orthogonal to the common line. Further, we know that $\angle(l_1(0), l_2(0))=0$, and $\angle(l_1(\pi/2), l_2(\pi/2))=\theta$.

For each $\alpha$, let $g(\alpha;\theta) = \angle(l_1(\alpha), l_2(\alpha))$ where $g(0)=0$ and $g(\pi/2)=\theta$.
Furthermore, we can derive a more general formula
$$g(\alpha;\theta)=\cos^{-1}(\cos^2\alpha+\sin^2\alpha\cos\theta)$$ from the spherical law of cosines. See the illustration below.

\begin{figure}[h]
    \centering
    \includegraphics[width=0.5\linewidth]{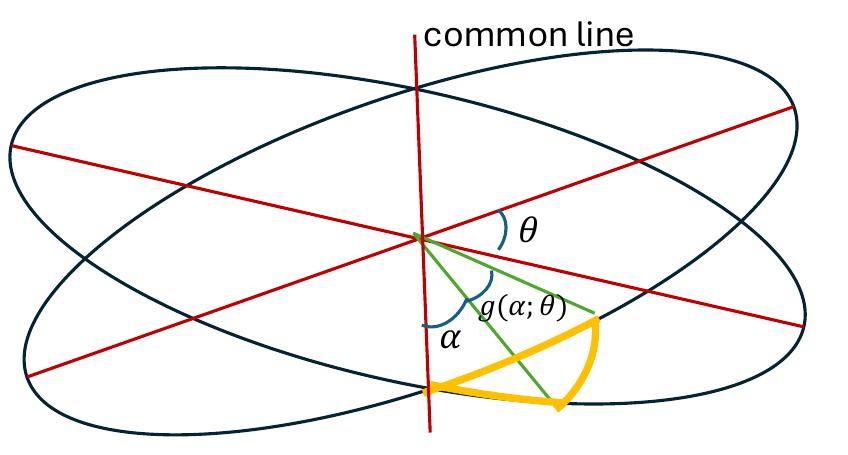}
    \caption{The formula of $g(\alpha; \theta)$ can be found by applying the spherical law of cosines to the yellow triangle on $S^2$.}
    \label{fig:enter-label}
\end{figure}


\noindent We note that
\begin{align}\label{eq:cosg}
\cos(g(\alpha;\theta))=\cos ^2 \alpha+\sin ^2 \alpha \cos \theta=\cos ^2 \alpha+\sin^2\alpha -\sin^2\alpha+\sin ^2 \alpha \cos \theta=1-\sin ^2 \alpha(1-\cos \theta).
\end{align}
Next, due to the known identity $1-\cos(y)=2\sin^2(y/2)$ and by letting both LHS and RHS equal to $x$, we have
$$
y=\cos ^{-1}(1-x)=2 \sin ^{-1}(\sqrt{x / 2}),
$$
which yields a trigonometric identity that involves inverse sine and cosine. Now, by applying this identity,
$$
g(\alpha;\theta)=\cos ^{-1}\left(1-\sin ^2 \alpha(1-\cos \theta)\right)=2 \sin ^{-1}\left(\sqrt{\frac{\sin ^2 \alpha(1-\cos \theta)}{2}}\right)=2 \sin ^{-1}\left(\sin \alpha \sin\left(\frac{\theta}{2}\right)\right).
$$
With this simplified $g(\alpha;\theta)$, we investigate its properties:
For the function $$f(x;k)=\sin^{-1}(k\sin x) ,$$
we know that $f'(x;k)=\frac{k\cos x}{\sqrt{1-k^2\sin^2 x}}$ and 

\begin{align}
f''(x;k)
=\frac{k(k^2-1)\sin x}{(1-k^2\sin^2 x)^{3/2}}\leq 0 ,
\end{align}
for $0\leq k\leq 1$. Therefore, $f(x;k)$ is a concave function of $x$, and it is upper bounded by the linear function $y=f'(0;k)x=kx$ and $f(x;k)$ reaches the maximum when $x=\pi/2$. Thus, by letting $k=\sin(\theta/2)$,
$$g(\alpha;\theta)\leq \min\left\{2\sin\left(\frac{\theta}{2} \right)\alpha,\, \theta\right\}\leq \min\left\{\alpha,\, 1\right\}\theta.$$
With this convenient upper bound, we can control the average of $g(\alpha;\theta)$ over $0\leq \alpha\leq \frac{\pi}{2}$:
$$\frac{2}{\pi}\int_0^{\frac{\pi}{2}}g^p(\alpha;\theta)d\alpha\leq \frac{2}{\pi}\left(\int_0^{1}\alpha^p \theta^p d\alpha+\int_1^{\frac{\pi}{2}} \theta^p d\alpha\right)=\frac{2}{\pi}\left(\frac{1}{p+1}\theta^p+(\frac\pi2-1)\theta^p\right)=\left(1-\frac{2p}{(p+1)\pi}\right)\theta^p.$$
\\
\\
We will then use the following lemma.
\begin{lemma}\label{lem:wtheta}
    $l_1$ and $l_2$ differ by angle $\theta$ and $f$ is any 2-D function supported on the unit disk, then:
$$
d_{W_p}\left(P_{l_1}f, P_{l_2}f\right)\leq \theta.
$$
\end{lemma}
\begin{proof}
Recall that $l_1$ and $l_2$ differ by angle $\theta$ and $f$ is any 2-D function supported on the unit disk, then:

\begin{align}
d_{W_p}^p\left(P_{l_1}f, P_{l_2}f\right) =d_{W_p}^p\left(P_{l_1}f, P_{l_1}(R_\theta f)\right)=\inf _{ \Gamma_P} \int_{\mathbb R\times \mathbb R}|x_1-x_2|^p d\Gamma_P(x_1, x_2)
\end{align}
where $\Gamma_p$ is the joint density with two marginals $P_{l_1}f$ and $P_{l_1}(R_\theta f)$. On the other hand, in the 2-D space,

\begin{align}
d_{W_p}^p(f, R_\theta f) = \inf _{ \Gamma} \int_{\mathbb R^2\times \mathbb R^2}\|(x_1, y_1)-(x_2, y_2)\|^p d\Gamma((x_1, y_1), (x_2, y_2))
\end{align}
where $\Gamma$ is the joint density of the two 2-D random variables whose marginal distributions are $f$ and $R_\theta f$ respectively. For now, WLOG assume that $l_1$ is the horizontal line, and $P_{l_1}$ takes the marginal over the $y$ coordinate, then:

\begin{align}
d_{W_p}^p(f, R_\theta f) \geq  \inf _{ \Gamma} \int_{\mathbb R^2\times \mathbb R^2}|x_1-x_2|^p d\Gamma((x_1, y_1), (x_2, y_2))= \inf _{ \Gamma_P} \int_{\mathbb R\times \mathbb R}|x_1-x_2|^p d\Gamma_P(x_1, x_2)= d_{W_p}^p\left(P_{l_1}f, P_{l_1}(R_\theta f)\right) ,
\end{align}
where the first equality is obtained by marginalizing the density function of the distribution $\Gamma$ over $y_1, y_2$ and applying Fubini's Theorem. Note that the above result coincides with Lemma 1 of~\cite{rao_wasserstein_2020} in the $W_1$ case, whereas we have more general $W_p$ but with a lower dimension.
\\
\\
Now, consider the Monge formulation of $W_p$, where we know:
$$
d_{W_p}^p(f, R_\theta f) \leq \int_D\|(x,y)-M((x,y))\|^p d F ,
$$
where $D$ is the unit disk and $F$ the distribution corresponding to the density $f$, and $M$ is any transportation map between $f$ and $R_\theta f$.
Then:
$$
d_{W_p}^p\left(f, R_\theta f\right) \leq \int_D\left\|(x, y)^\top-R_\theta(x, y)^\top\right\|^p d F .
$$
Since for any vector $(x, y)$ with $\|(x, y)\| \leq 1$, we have
$$
\left\|(x, y)^\top-R_\theta(x, y)^\top\right\| \leq 2 \sin (\theta / 2)\leq \theta ,
$$
and therefore:
$$
d_{W_p}^p\left(P_{l_1}f, P_{l_2}f\right)\leq d_{W_p}^p(f, R_\theta f) \leq \int_D\theta^p dF=\theta^p.
$$

\end{proof}
In the sliced Wasserstein case, we have:

\begin{align}
d_{W_p}(P_{l_1(\alpha)}I_1, P_{l_2(\alpha)}I_2) \leq g(\alpha;\theta).
\end{align}
By the definition of the sliced $p$-Wasserstein distance, we have:
\begin{align}
\min_{\alpha}d_{SW_p}^p(I_1, \mathbf R_\alpha I_2)\leq \frac{2}{\pi}\int_0^\frac{\pi}{2} d_{W^p}^p(P_{l_1(\alpha)}I_1, P_{l_2(\alpha)}I_2)d\alpha\leq  \frac{2}{\pi}\int_0^{\frac{\pi}{2}}g^p(\alpha;\theta)d\alpha\leq \left(1-\frac{2p}{(p+1)\pi}\right)\theta^p.
\end{align}

\subsection{Proof for the stability to translation}
\begin{lemma}
Let $P_{l_\theta}f$ be the projection of $f$ on the line $l_\theta$ where $\theta=\angle(\bfs, l_\theta)$, then:
$$d_{W_p}(P_{l_\theta}f, P_{l_\theta}T_{\bfs}f) = \cos \theta \|\bfs\| .$$ 
\end{lemma}
This result is obvious as the 1-D functions $P_{l_\theta}f$ and $P_{l_\theta}g$ are also shifted version of each other, where the shift magnitude is $\cos\theta\|\bfs\|$. The lemma directly follows from the translation-equivariant property of the $p$-Wasserstein distance.
The theorem is then concluded as follows:
$$d_{SW_p}^p(f, T_{\bfs}f)=\frac{2}{\pi}\int_0^\frac{\pi}{2} d_{W_p}^p(P_{l_\theta}f_1, P_{l_\theta}f_2)d\theta= \frac{2}{\pi}\|\bfs\|^p\int_0^{\frac{\pi}{2}} \cos \theta^p  d\theta = \frac{2}{\pi}\|\bfs\|^p\frac{\sqrt{\pi}\,\Gamma\left(\frac{p+1}{2}\right)}{2\,\Gamma\left(\frac{p}{2}+1\right)}=\frac{\Gamma\left(\frac{p+1}{2}\right)}{\sqrt{\pi}\Gamma\left(\frac{p}{2}+1\right)}\|\bfs\|^p.$$

\section{Ramp-filtered sliced Wasserstein distances: stability to translation and viewing direction}
\label{sec:rfsw_stability}

For $\theta\in[0,2\pi)$ we denote by $P_\theta f$ the 1-D line projection of $f$ (Radon slice) along direction $\theta$,
supported on $[-1,1]$.
The \emph{ramp filter} $h$ is the 1-D distribution whose Fourier multiplier is $\widehat h(\xi)=|\xi|$,
so that $h*u=\partial_s(\mathcal{H}u)$, where $\mathcal{H}$ is the 1-D Hilbert transform in the $s$-variable.
We define the ramp-filtered slice
\[
\tilde\mu_\theta:=h*P_\theta f,\qquad
\tilde\nu_\theta:=h*P_\theta g.
\]
Because $\int \tilde\mu_\theta=0=\int \tilde\nu_\theta$, each slice is signed.
We therefore split into \emph{positive} and (absolute) \emph{negative} parts
\[
\mu_\theta^{+} := (\tilde\mu_\theta)_{+}=\max\{\tilde\mu_\theta,0\},\qquad
\mu_\theta^{-} := (\tilde\mu_\theta)_{-}=\max\{-\tilde\mu_\theta,0\},
\]
and similarly $\nu_\theta^{\pm}$ from $\tilde\nu_\theta$; each $\mu_\theta^{\pm}$ and $\nu_\theta^{\pm}$ is nonnegative.
We normalize each part to a probability density on $[-1,1]$:
\[
\bar\mu_\theta^{\pm}:=\frac{\mu_\theta^{\pm}}{\|\mu_\theta^{\pm}\|_{L^1}},\qquad
\bar\nu_\theta^{\pm}:=\frac{\nu_\theta^{\pm}}{\|\nu_\theta^{\pm}\|_{L^1}},
\]
with the convention that if a part vanishes, the corresponding Wasserstein distance term is taken to be $0$.
The \emph{ramp-filtered sliced $p$-Wasserstein distance} used in the main text is
\[
d_{RF\!SW_p}^p(f,g)
:=\frac{1}{2\pi}\int_0^{2\pi}\!\Big(d_{W_p}^p(\bar\mu_\theta^{+},\bar\nu_\theta^{+})
+d_{W_p}^p(\bar\mu_\theta^{-},\bar\nu_\theta^{-})\Big)\,d\theta .
\]

\subsection{Stability to translation (all $p\ge 1$)}
\begin{theorem}[Translation equivariance of $d_{RF\!SW_p}$]
\label{thm:rfsw_translation}
Let $p\ge 1$ and $g=T_{\mathbf{s}}f$ be a translated of $f$ by $\mathbf{s}\in\mathbb{R}^2$:
$T_{\mathbf{s}}f(x)=f(x-\mathbf{s})$. Then:
\[
d_{RF\!SW_p}^p(f,g)
= \frac{4}{\pi}\,\Bigg(\int_0^{\pi/2}\!\!\cos^p\varphi\,d\varphi\Bigg)\,\|\mathbf{s}\|^p
= \frac{2\,\Gamma\!\big(\frac{p+1}{2}\big)}{\sqrt{\pi}\,\Gamma\!\big(\frac{p}{2}+1\big)}\,\|\mathbf{s}\|^p .
\]
Equivalently, $d_{RF\!SW_p}(f,T_{\mathbf{s}}f)$ is a constant multiple of $\|\mathbf{s}\|$, with exactly twice the constant of the unfiltered $d_{SW_p}$.  
\end{theorem}

\begin{proof}
The projection commutes with translations in the sense
$P_\theta(T_{\mathbf{s}}f)(t)=P_\theta f(t-\mathbf{s}\cdot\theta)$.
Since $h$ is translation-invariant, we also have
\[
h*P_\theta(T_{\mathbf{s}}f)(t)= \big(h*P_\theta f\big)(t-\mathbf{s}\cdot\theta)=\tilde\mu_\theta(t-\Delta_\theta),
\quad \Delta_\theta:=\mathbf{s}\cdot\theta .
\]
Taking positive and negative parts and then normalizing preserves this shift:
$\bar\nu_\theta^{\pm}=\tau_{\Delta_\theta}\bar\mu_\theta^{\pm}$, where $(\tau_{\Delta}\phi)(t):=\phi(t-\Delta)$.
For a probability density $\rho$ on $\mathbb{R}$ and $\Delta\in\mathbb{R}$, the optimal transport map from $\rho$ to $\tau_\Delta\rho$
is the translation $t\mapsto t+\Delta$, hence $d_{W_p}(\rho,\tau_\Delta\rho)=|\Delta|$ for all $p\ge 1$.
Therefore for each $\theta$,
\[
d_{W_p}^p(\bar\mu_\theta^{+},\bar\nu_\theta^{+})=|\mathbf{s}\cdot\theta|^p
\quad\text{and}\quad
d_{W_p}^p(\bar\mu_\theta^{-},\bar\nu_\theta^{-})=|\mathbf{s}\cdot\theta|^p .
\]
Averaging over $\theta$ and using
$\frac{1}{2\pi}\int_0^{2\pi}\!|\mathbf{s}\cdot\theta|^p\,d\theta
=\frac{2}{\pi}\|\mathbf{s}\|^p\int_0^{\pi/2}\!\cos^p\varphi\,d\varphi$
gives the stated closed form via the Beta-function identity
$\int_0^{\pi/2}\cos^p\varphi\,d\varphi
=\tfrac{\sqrt{\pi}}{2}\,\Gamma\!\big(\frac{p+1}{2}\big)/\Gamma\!\big(\frac{p}{2}+1\big)$.
\end{proof}

\subsection{Ramp-filtered stability to changes of viewing directions}
\label{sec:rf_viewing_general_p}

Throughout this section $Z:\mathbb{R}^3\to[0,\infty)$ is a probability density supported on the closed unit ball
\[
B_3=\{x\in\mathbb{R}^3:\ \|x\|_2\le 1\},\qquad \int_{B_3} Z(x)\,dx=1.
\]
Let $a_1,a_2\in\mathbb{S}^2$ be two viewing directions with $\angle(a_1,a_2)=\theta\in[0,\pi]$,
and let $I_1,I_2:\mathbb{R}^2\to[0,\infty)$ be the corresponding $2$-D projections of $Z$ (i.e., integrals of $Z$ along lines
parallel to $a_1$ and $a_2$ respectively). Note that here we avoid notation $f,g$ due to the conflict to $g(\alpha; \theta)$. For an in-plane angle $\alpha\in[0,\pi/2]$ we denote by
$l_1(\alpha)$ and $l_2(\alpha)$ the lines in the two image planes that make angle $\alpha$ with the common line;
their in-plane misalignment is
\[
g(\alpha;\theta)=\cos^{-1}\!\big(\cos^2\alpha+\sin^2\alpha\cos\theta\big)
=2\,\sin^{-1}\!\big(\sin\alpha\,\sin(\theta/2)\big),
\]
with $g(0;\theta)=0$, $g(\tfrac{\pi}{2};\theta)=\theta$ and $g(\alpha;\theta)\le \min\{\alpha\,\theta,\ \theta\}$.

On each line $l$ in an image plane, the 1-D \emph{ramp filter} is applied to the line projection as a spatial convolution
by a kernel $h$ whose Fourier multiplier is $|\xi|$. Equivalently,
\[
h*u=\partial_s(\mathcal{H}u),
\]
where $s$ is the line coordinate and $\mathcal H$ is the 1-D Hilbert transform. For any 1-D signal $w$ we write
$w_+:=\max\{w,0\}$ and $w_-:=\max\{-w,0\}$, and when $\|w_\pm\|_{L^1}>0$ we denote the normalized parts by
$w_\pm^\sharp:=w_\pm/\|w_\pm\|_{L^1}$ (when both sign-masses vanish we set the corresponding term below to $0$).

The \emph{ramp-filtered sliced $p$-Wasserstein distance} between two 2-D projections $I_1,I_2$ is
\[
d^p_{RF\!SW_p}(I_1,I_2)
=\frac{2}{\pi}\int_{0}^{\pi/2}\Big(
d^p_{W_p}\big((h*P_{l_1(\alpha)}I_1)_+^\sharp,(h*P_{l_2(\alpha)}I_2)_+^\sharp\big)
+
d^p_{W_p}\big((h*P_{l_1(\alpha)}I_1)_-^\sharp,(h*P_{l_2(\alpha)}I_2)_-^\sharp\big)
\Big)\,d\alpha,
\]
where $P_l$ extracts the 1-D projection along the line $l$.

The next theorem is the ramp-filtered analogue of Theorem~\ref{thm:view} in the unfiltered case. It holds for all
$p\in[1,\infty)$ and \emph{does not} assume any density floor for the normalized signed parts.

\begin{theorem}[Ramp-filtered viewing-direction stability, all $p\ge 1$]
\label{thm:rf_view_general_p}
Assume $Z\in W^{1,2}(B_3)$.
Let $I_1$ and $I_2$ be the orthographic projections of $Z$ along $a_1$ and $a_2$, with $\angle(a_1,a_2)=\theta$.
For $\alpha\in[0,\pi/2]$ define the line masses
\[
m_{I_1,\pm}(\alpha):=\big\|(h*P_{l_1(\alpha)}I_1)_\pm\big\|_{L^1},\qquad
m_{I_2,\pm}(\alpha):=\big\|(h*P_{l_2(\alpha)}I_2)_\pm\big\|_{L^1},
\]
and the ``matched'' masses
\[
m_\pm^*(\alpha):=\min\{\,m_{I_1,\pm}(\alpha),\ m_{I_2,\pm}(\alpha)\,\}.
\]
(When both $m_{I_1,\pm}(\alpha)$ and $m_{I_2,\pm}(\alpha)$ are zero we set $1/m_\pm^*(\alpha):=0$; if only one is zero, the
corresponding integrand below is interpreted as $+\infty$, making the bound trivial.)
Then, for every $p\in[1,\infty)$,
\begin{equation}
\label{eq:rf_main_general_p}
\min_{\beta\in\mathbb{R}}\ d^p_{RF\!SW_p}\!\big(I_1,\ R_\beta I_2\big)
\ \le\
C_p(Z)\,\frac{2}{\pi}\int_{0}^{\pi/2}
\Big(\frac{1}{m_+^*(\alpha)}+\frac{1}{m_-^*(\alpha)}\Big)\,g(\alpha;\theta)\ d\alpha,
\end{equation}
where
\[
C_p(Z):=2^{p}\,\sqrt{\frac{8\pi}{3}}\ \|\nabla Z\|_{L^2(B_3)}.
\]
In particular, using $g(\alpha;\theta)\le \theta$,
\begin{equation}
\label{eq:rf_main_general_p_simple}
\min_{\beta} d_{RF\!SW_p}\!\big(I_1,\ R_\beta I_2\big)
\ \le\
\Big(C_p \mathcal{M}\,\theta\ \Big)^{1/p},
\qquad
\text{where } \mathcal{M}(I_1,I_2):=\frac{2}{\pi}\int_{0}^{\pi/2}\Big(\frac{1}{m_+^*(\alpha)}+\frac{1}{m_-^*(\alpha)}\Big)\,d\alpha.
\end{equation}
\end{theorem}

\noindent
\textit{Comments.} (a) No density floor is assumed; the image-dependent factor $\mathcal{M}(I_1,I_2)$ captures the
sensitivity introduced by sign-splitting and renormalization.  
(b) If $m_\pm^*(\alpha)\ge m_0>0$ for all $\alpha$ (a nondegeneracy that \emph{may} hold in practice), then
$\mathcal{M}(I_1,I_2)\le 2/m_0$ and the right-hand side of \eqref{eq:rf_main_general_p_simple} is $O(\theta^{1/p})$
with an explicit constant.  
(c) When $m_\pm^*$ are essentially independent of $\alpha$ one may replace $g(\alpha;\theta)$ in
\eqref{eq:rf_main_general_p} by its average to get the sharper factor
$\frac{2}{\pi}\int_0^{\pi/2}g(\alpha;\theta)\,d\alpha=(1-\frac{1}{\pi})\,\theta$.

\subsection*{Proof of Theorem~\ref{thm:rf_view_general_p}}

\paragraph{Step 1: an $L^2$ bound for the ramp-filtered difference along a fixed pair of in-plane lines.}
Fix $\alpha\in[0,\pi/2]$ and rotate $I_2$ by an in-plane angle $\beta$ so that the common lines coincide
(this does not increase the $\alpha$-average). Let $v_1,v_2\in\mathbb{S}^2$ be the unit directions of $l_1(\alpha)$ and
$l_2(\alpha)$ in their respective planes, and complete them to orthonormal triples $(v_i,w_i,a_i)$, $i=1,2$,
with $a_1=a_1$ and $a_2=a_2$ the two viewing directions.
Define the 1-D line integrals
\[
u(s):=P_{l_1(\alpha)}I_1(s)=\int_{\mathbb{R}}\int_{\mathbb{R}} Z\big(sv_1+t\,w_1+r\,a_1\big)\,dr\,dt,
\]
\[
v(s):=P_{l_2(\alpha)}I_2(s)=\int_{\mathbb{R}}\int_{\mathbb{R}} Z\big(sv_2+t\,w_2+r\,a_2\big)\,dr\,dt.
\]
Differentiating under the integral (justified by $Z\in W^{1,2}(B_3)$ and compact support),
\[
\partial_s u(s)=\int\!\!\int (\nabla Z)\big(sv_1+t\,w_1+r\,a_1\big)\cdot v_1\ dr\,dt,
\quad
\partial_s v(s)=\int\!\!\int (\nabla Z)\big(sv_2+t\,w_2+r\,a_2\big)\cdot v_2\ dr\,dt.
\]
Subtracting and applying Cauchy–Schwarz in $(t,r)$ for each fixed $s$ gives
\[
|\partial_s u(s)-\partial_s v(s)|
\le \Big(\int\!\!\int \big|(\nabla Z)(\cdot)\big|^2\,dr\,dt\Big)^{1/2}
     \Big(\int\!\!\int 1_{D_s}\,dr\,dt\Big)^{1/2}\,\|v_1-v_2\|,
\]
where $D_s$ is the disc $\{(t,r): t^2+r^2\le 1-s^2\}$ in the $(t,r)$-plane and $\|v_1-v_2\|\le g(\alpha;\theta)$.
Squaring and integrating in $s\in[-1,1]$,
\[
\int_{-1}^1 |\partial_s u(s)-\partial_s v(s)|^2\,ds
\ \le\ \Big(\int_{B_3}|\nabla Z(x)|^2\,dx\Big)\,
      \Big(\int_{-1}^1 \pi(1-s^2)\,ds\Big)\,g(\alpha;\theta)^2
\ =\ \frac{4\pi}{3}\,\|\nabla Z\|_{L^2(B_3)}^2\,g(\alpha;\theta)^2.
\]
Since $h* w=\partial_s(\mathcal H w)$ and $\mathcal H$ is unitary on $L^2(\mathbb{R})$,
\begin{equation}
\label{eq:L2_diff_ramp}
\|h*u-h*v\|_{L^2([-1,1])}=\|\partial_s u-\partial_s v\|_{L^2([-1,1])}
\ \le\ \sqrt{\frac{4\pi}{3}}\ \|\nabla Z\|_{L^2(B_3)}\,g(\alpha;\theta).
\end{equation}
By Cauchy–Schwarz on $[-1,1]$,
\begin{equation}
\label{eq:L1_from_L2}
\|h*u-h*v\|_{L^1([-1,1])}\ \le\ \sqrt{2}\ \|h*u-h*v\|_{L^2([-1,1])}
\ \le\ \sqrt{\frac{8\pi}{3}}\ \|\nabla Z\|_{L^2}\,g(\alpha;\theta).
\end{equation}

\paragraph{Step 2: from the $L^1$ bound to $d_{W_p}$ for normalized sign parts.}
For any functions $a,b$ on $[-1,1]$,
\[
\|a_+-b_+\|_{L^1}\le \|a-b\|_{L^1},\qquad \|a_- - b_-\|_{L^1}\le\|a-b\|_{L^1}.
\]
If $m_a:=\|a_+\|_{L^1}>0$ and $m_b:=\|b_+\|_{L^1}>0$, then
\[
\left\|\frac{a_+}{m_a}-\frac{b_+}{m_b}\right\|_{L^1}
\le \frac{2}{\min\{m_a,m_b\}}\,\|a_+-b_+\|_{L^1}
\le \frac{2}{\min\{m_a,m_b\}}\,\|a-b\|_{L^1}.
\]
On $[-1,1]$ (diameter $2$) and for any $p\ge 1$,
\[
d_{W_p}^p(\rho,\sigma)\ \le\ 2^{p-1}\,d_{W_1}(\rho,\sigma)\ \le\ 2^{p-1}\,\|\rho-\sigma\|_{L^1}.
\]
Applying these three facts with $a=h*u$ and $b=h*v$ yields, for the positive parts,
\[
d_{W_p}^{p}\big((h*u)_+^\sharp,(h*v)_+^\sharp\big)
\ \le\ \frac{2^{p}}{\min\{\|a_+\|_{L^1},\|b_+\|_{L^1}\}}\ \|a-b\|_{L^1},
\]
and similarly for the negative parts. Using \eqref{eq:L1_from_L2} and denoting
$m_+^*(\alpha):=\min\{\,\|(h*u)_+\|_{L^1},\ \|(h*v)_+\|_{L^1}\,\}$ and
$m_-^*(\alpha):=\min\{\,\|(h*u)_-\|_{L^1},\ \|(h*v)_-\|_{L^1}\,\}$, we obtain
\begin{equation}
\label{eq:per_alpha_bound}
\begin{aligned}
&d_{W_p}^{p}\big((h*u)_+^\sharp,(h*v)_+^\sharp\big)
+d_{W_p}^{p}\big((h*u)_-^\sharp,(h*v)_-^\sharp\big)\\
&\qquad\le\ 2^{p}\,\sqrt{\frac{8\pi}{3}}\ \|\nabla Z\|_{L^2(B_3)}\
\Big(\frac{1}{m_+^*(\alpha)}+\frac{1}{m_-^*(\alpha)}\Big)\,g(\alpha;\theta).
\end{aligned}
\end{equation}

\paragraph{Step 3: average over $\alpha$ and minimize over in-plane rotations.}
By definition of $d_{RF\!SW_p}$, integrating \eqref{eq:per_alpha_bound} against $\tfrac{2}{\pi}\,d\alpha$ over $[0,\pi/2]$
gives, for the \emph{aligned} choice of in-plane rotation,
\[
d_{RF\!SW_p}(I_1,R_\beta I_2)^{p}
\ \le\
2^{p}\,\sqrt{\frac{8\pi}{3}}\ \|\nabla Z\|_{L^2}\,
\frac{2}{\pi}\int_{0}^{\pi/2}
\Big(\frac{1}{m_+^*(\alpha)}+\frac{1}{m_-^*(\alpha)}\Big)\,g(\alpha;\theta)\,d\alpha.
\]
Since $\min_{\beta} d_{RF\!SW_p}(I_1,R_\beta I_2)\le d_{RF\!SW_p}(I_1,R_\beta I_2)$ for any $\beta$, this proves
\eqref{eq:rf_main_general_p} with $C_p(Z)=2^{p}\sqrt{\tfrac{8\pi}{3}}\ \|\nabla Z\|_{L^2(B_3)}$.
Finally, using $g(\alpha;\theta)\le\theta$ yields \eqref{eq:rf_main_general_p_simple}.
\hfill$\square$

The bound controls $d_{RF\!SW_p}$ by a constant times $\theta^{1/p}$
(see \eqref{eq:rf_main_general_p_simple}). This is weaker than the linear-in-$\theta$ rate for the \emph{unfiltered}
$d_{SW_p}$ in Theorem~\ref{thm:view}; the loss stems from (i) the high-pass ramp filter (which converts a derivative-level
difference of line integrals into the metric), and (ii) the nonlinear sign-splitting plus renormalization, which forces a
passage through an $L^1$ estimate before invoking a $W_p$ bound on $[-1,1]$.
\\
\textbf{The final bound: } In fact, we can show that in the next section $m_\pm^*(\alpha)\ge m_0>0$ for all $\alpha$
(a nondegenerate contrast condition for the filtered slices), then
\[
\min_{\beta} d_{RF\!SW_p}(I_1,R_\beta I_2)
\ \le\
\Big( 2^{p+1}\sqrt{\tfrac{8\pi}{3}}\,\tfrac{\|\nabla Z\|_{L^2}}{m_0}\,
\Big(1-\tfrac{1}{\pi}\Big)\,\theta\Big)^{\!1/p},
\]
using the exact average $\tfrac{2}{\pi}\int_0^{\pi/2} g(\alpha;\theta)\,d\alpha=(1-\tfrac{1}{\pi})\theta$. We show in the next section that $m_0$ is positive and only depends on $Z$.

\subsection{Discussion on the lower bound on the ramp–filtered signed mass ($m_0$).}
We work with the continuous Fourier transform
\[
\widehat{f}(\xi)=\int_{\mathbb{R}} f(s)\,e^{-i\xi s}\,ds,
\qquad
f(s)=\frac{1}{2\pi}\int_{\mathbb{R}}\widehat{f}(\xi)\,e^{i\xi s}\,d\xi .
\]
Let $Z:\mathbb{R}^3\to[0,\infty)$ be a probability density supported in the unit ball, and let $I$ be its
$2$-D orthographic projection along some view $a\in\mathbb{S}^2$.  For a unit line $l\subset a^\perp$,
write $\mu:=P_l I$ and let $w:=h*\mu$ denote the (continuous) ramp–filtered line signal, so that
\[
\widehat{w}(\xi)=|\xi|\,\widehat{\mu}(\xi)=|\xi|\,\widehat{Z}(\xi l)
\qquad(\text{Fourier–slice theorem}).
\]
Since $|\xi|\,\widehat{Z}(\xi l)$ vanishes at $\xi=0$, $\int_{\mathbb{R}}w=0$ and hence
$\|w_+\|_{L^1}=\|w_-\|_{L^1}=\tfrac12\|w\|_{L^1}$.

The key point is that we can extract a \emph{purely continuous} lower bound for $\|w_+\|_{L^1}$ directly from
the line–slice $\widehat{Z}(\xi l)$, with no discretization.  We do this by \emph{analytic low–pass
smoothing} via the Fejér kernel, which is a nonnegative continuous kernel of unit mass.

\medskip
\noindent\textbf{Fejér low–pass on $\mathbb{R}$.}
For $R>0$ define the triangular frequency multiplier
\[
\phi_R(\xi):=\bigl(1-|\xi|/R\bigr)_+ \quad\text{and}\quad
\kappa_R(s):=\frac{1}{2\pi}\int_{\mathbb{R}} \phi_R(\xi)\,e^{i\xi s}\,d\xi .
\]
Then $\kappa_R\ge 0$, $\int_{\mathbb{R}}\kappa_R=1$, and $\widehat{\kappa_R}=\phi_R$.
Given any $w\in L^1(\mathbb{R})$ with $\int w=0$, set $w_R:=w*\kappa_R$.
Since $\kappa_R\ge 0$ and has unit mass,
\begin{equation}\label{eq:posmass_monotone}
\|w_+\|_{L^1}\ \ge\ \|(w*\kappa_R)_+\|_{L^1}\,,
\qquad
\|w_-\|_{L^1}\ \ge\ \|(w*\kappa_R)_-\|_{L^1},
\end{equation}
and $\int w_R=0$, hence $\|(w_R)_+\|_{L^1}=\|(w_R)_-\|_{L^1}=\tfrac12\|w_R\|_{L^1}$.

\medskip
\noindent\textbf{Bandlimited $L^1$–lower bound.}
If $\widehat{f}$ is supported in $[-R,R]$, then
\begin{equation}\label{eq:L1_lower_bandlimited}
\|f\|_{L^1}\ \ge\ \frac{\|f\|_{L^2}^2}{\|f\|_{L^\infty}}
\ \ge\ \frac{\tfrac{1}{2\pi}\|\widehat{f}\|_{L^2}^2}{\tfrac{1}{2\pi}\|\widehat{f}\|_{L^1}}
\ \ge\ \frac{1}{\sqrt{2R}}\ \|\widehat{f}\|_{L^2},
\end{equation}
where we used Parseval ($\|f\|_2^2=\tfrac{1}{2\pi}\|\widehat{f}\|_2^2$), $\|f\|_\infty\le \tfrac{1}{2\pi}\|\widehat{f}\|_1$,
and Cauchy–Schwarz on $[-R,R]$.

Applying \eqref{eq:L1_lower_bandlimited} to $f=w_R$ (whose spectrum is
$\widehat{w_R}(\xi)=\phi_R(\xi)\widehat{w}(\xi)$ and is supported in $[-R,R]$) gives
\[
\|w_R\|_{L^1}\ \ge\ \frac{1}{\sqrt{2R}}\ \|\phi_R\,\widehat{w}\|_{L^2}
=\frac{1}{\sqrt{2R}}\left(\int_{\mathbb{R}} \phi_R(\xi)^2\,\xi^2\,|\widehat{Z}(\xi l)|^2\,d\xi\right)^{\!1/2}.
\]
Combining with \eqref{eq:posmass_monotone} and $\|(w_R)_+\|_1=\tfrac12\|w_R\|_1$ yields, \emph{for every line $l$},
\begin{equation}\label{eq:line_mass_lower_exact}
\|(h*P_l I)_+\|_{L^1}
\ \ge\ \frac{1}{2\sqrt{2R}}\left(\int_{\mathbb{R}} \phi_R(\xi)^2\,\xi^2\,|\widehat{Z}(\xi l)|^2\,d\xi\right)^{\!1/2}.
\end{equation}
Since $\phi_R(\xi)\ge \tfrac12$ on $|\xi|\le R/2$, the right-hand side admits the simpler bound
\begin{equation}\label{eq:line_mass_lower_halfband}
\|(h*P_l I)_+\|_{L^1}
\ \ge\ \frac{1}{4\sqrt{2R}}\left(\int_{-R/2}^{R/2} \xi^2\,|\widehat{Z}(\xi l)|^2\,d\xi\right)^{\!1/2}.
\end{equation}

\medskip
\noindent\textbf{Uniform mass floor for $m_0$.}
Define the \emph{line–band energy}
\[
\mathcal{E}_Z(R):=\inf_{l\in S^2}
\ \int_{-R/2}^{R/2} \xi^2\,|\widehat{Z}(\xi l)|^2\,d\xi .
\]
Assume $\mathcal{E}_Z(R)>0$ for some $R>0$.  Then, taking the infimum of
\eqref{eq:line_mass_lower_halfband} over all lines in both planes gives
\begin{equation}\label{eq:m0_fourier_continuous}
m_0\ :=\ \inf_{\alpha\in[0,\pi/2]}\ \min\Big\{
\|(h*P_{l_1(\alpha)}I_1)_+\|_{L^1},\ \|(h*P_{l_1(\alpha)}I_1)_-\|_{L^1},\
\|(h*P_{l_2(\alpha)}I_2)_+\|_{L^1},\ \|(h*P_{l_2(\alpha)}I_2)_-\|_{L^1}\Big\}
\ \ge\ \frac{1}{4\sqrt{2R}}\ \sqrt{\ \mathcal{E}_Z(R)\ }.
\end{equation}

\paragraph{Remarks.}
\begin{itemize}
\item[(1)] The parameter $R>0$ is arbitrary.  The bound \eqref{eq:m0_fourier_continuous} holds for every $R$,
and one may optimize $R$ to maximize the right-hand side if additional information on the line–band energies
$|\widehat{Z}(\xi l)|$ is available.

\item[(2)] The lower bound depends only on $Z$.
\end{itemize}

\section{Rotational alignment - extended results}
\label{app:alignment_extended}

To expand the numerical results in~\Cref{sec:alignment}, we additionally test other established computational methods for approximating the Wasserstein and sliced Wasserstein distances, and compare them against the timing and alignment accuracy of our approach. Computing the sliced Wasserstein distance in~\cref{eq:sliced_wasserstein} is often intractable and must be approximated. Equivalently,~\cref{eq:sliced_wasserstein} can be written as:

\begin{equation}
\label{eq:sliced_wasserstein_expectation}
d_{SW_p}^p (f,g) = \mathbb{E}_{\theta \sim \mathcal{U}(\mathbb{S}^{d-1})}\left[ W_p^p(P_\theta f , P_\theta g) \right] .
\end{equation}
A standard method to approximate this quantity is to replace the expectation with a fixed number of random and uniform projections from the unit circle~\cite{sisouk_users_2025}. This is referred to as the Monte Carlo sliced Wasserstein distance:

\begin{equation}
\label{eq:monte_carlo_sliced_wasserstein}
d_{MCSW_p}^p(f,g) \approx \frac 1 n \sum_i^n d_{W_p}^p(P_{\theta_i} f, P_{\theta_i} g) \;\;\;, \;\;\; \theta_i \sim \mathcal{U}(\mathbb{S}^{d-1}) .
\end{equation}
In our setting of rigid alignment, a Monte Carlo style approach can easily be implemented by using $L$ random projection directions, and computing the discretized sliced 2-Wasserstein distance in~\cref{eq:discrete_rot_sw2}. Therefore, the computational complexity remains $\mathcal{O}(L^2 \log L)$ operations. The only difference between the Monte Carlo approach and our approach is the control over the projection angles. However, in the context of rigid alignment, random sampling of projection directions does not provide a benefit over the equispaced projection directions used in our algorithm (see alignment results in~\Cref{tab:summary}). 

Another popular metric related to the sliced Wasserstein distance is the max-sliced Wasserstein distance~\cite{deshpande_max_2019}. In our algorithm, where a fixed number of equispaced projections are generated from each image (i.e., the matrices $U$ and $V$), the discrete max-sliced 2-Wasserstein distance can be approximated as:


\begin{equation}
\label{eq:max_sliced_wasserstein}
d_{MSW_2}^2 (f, g) \approx \| U_I - V_I \|_{2, \infty}^2 ,
\end{equation}
where the norm is taken over the columns. Although this can be easily computed in our framework, the fast cross correlation in~\cref{eq:discrete_rot_sw2} can no longer be used, and so the total complexity for alignment over rotations reverts to $\mathcal{O}(L^3)$ operations. 

Fast approximations of the Wasserstein distance in~\cref{eq:p_wasserstein} rely on the addition of an entropic regularization term (see~\cite{cuturi_sinkhorn_2013} for more detail). The entropic regularized Wasserstein distance is:

\begin{equation}
\label{eq:sinkhorn_distance}
d_{SD_2}^2 (f,g) = \min_{\gamma \in \Gamma(f,g)} \left[ \int_{\Omega\times \Omega} \|\mathbf{x_1} - \mathbf{x_2}\|^2 d\gamma(\mathbf{x_1},\mathbf{x_2}) + \lambda \, \mbox{KL}(\gamma || f \otimes g) \right] ,
\end{equation}
where KL is the Kullback-Leibler divergence, and $f \otimes g$ is the product measure. Efficient solving of~\cref{eq:sinkhorn_distance} makes use of the Sinkhorn–Knopp algorithm, and it is therefore sometimes referred to as the Sinkhorn distance. When $f$ and $g$ are discretized images on a uniform grid, and the cost function inside the integral is the squared Euclidean distance, the Sinkhorn distance can be accelerated further and is referred to as the convolutional Wasserstein distance (see~\cite{solomon_convolutional_2015} for more detail).

We show alignment results for all metrics in~\Cref{tab:summary}. However, we exclude the convolutional Wasserstein distance here as it provides no speed boost for the $39 \times 39$ MNIST images (see~\Cref{tab:dense_timing_rotations}). The experimental set up is exactly the same as in~\Cref{sec:alignment}. Here, we report a single value of the cumulative alignment up to $\pm 15 \degree$ to simplify the results. While the ramp-filtered sliced 2-Wasserstein distance is overall the best performing metric, we note that the max-sliced 2-Wasserstein distance performs particularly well when large shifts are present in the images. We additionally report the timing and computational complexities of each method in~\Cref{tab:all_timing}.


\begin{table}[!ht]

\caption{Rotated MNIST alignment results for various metrics. Values are reported as cumulative percent of digits aligned within $\pm 15 \degree$ of the ground truth rotation. Columns indicate the digit. The Sinkhorn distance was computed with a regularization term $\lambda = 0.01$ and $H = 3$ iterations.}
    \centering
    \small
    \begin{tabular}{|l|c|c|c|c|c|c|c|}
    \hline
        
        \multicolumn{8}{|c|}{\textbf{shift = 0 pixels}} \\ \hline
        ~ & \textbf{2} & \textbf{3} & \textbf{4} & \textbf{5} & \textbf{6} & \textbf{7} & \textbf{9} \\ \hline
        Euclidean Distance & 60.7 & 69.9 & 57.8 & 30.2 & 53.8 & 67.4 & 60 \\
        Sliced 2-Wasserstein Distance & 56.9 & 61.1 & 54.9 & 35.8 & 52.9 & 63.4 & 58 \\
        Ramp-Filtered Sliced 2-Wasserstein Distance & \cellcolor{green!30}73.3 & 54.7 & 59.8 & \cellcolor{green!30}46.4 & \cellcolor{green!30}71 & \cellcolor{green!30}79.8 & \cellcolor{green!30}70.5 \\
        Monte Carlo Sliced 2-Wasserstein Distance & 9.7 & 30.5 & 36.8 & 27.8 & 47.6 & 31.3 & 30.8 \\
        Max-Sliced 2-Wasserstein Distance & 33.5 & 54.2 & 49.2 & 29.7 & 52.9 & 57.5 & 58.6 \\
        Wavelet Earth Mover's Distance & 63.1 & \cellcolor{green!30}72.3 & 57 & 35.9 & 54.9 & 67.7 & 61.1 \\
        Sinkhorn Distance & 70.4 & 64 & 63 & 38.5 & 56.6 & 64.4 & 59.9 \\
        2-Wasserstein Distance & 58.1 & 62.6 & \cellcolor{green!30}66.4 & 37.7 & 56.4 & 64 & 60.3 \\ 
        \hline
        
        \multicolumn{8}{|c|}{\textbf{shift = 2 pixels}} \\ \hline
        ~ & \textbf{2} & \textbf{3} & \textbf{4} & \textbf{5} & \textbf{6} & \textbf{7} & \textbf{9} \\ \hline
        Euclidean Distance & 35.5 & 27.5 & 23.4 & 20.5 & 33.2 & 40.3 & 26.3 \\
        Sliced 2-Wasserstein Distance & 42.6 & 36 & 25 & 27.6 & 26 & 52 & 44.5 \\
        Ramp-Filtered Sliced 2-Wasserstein Distance & \cellcolor{green!30}66.7 & 42.2 & \cellcolor{green!30}50.4 & \cellcolor{green!30}44.4 & \cellcolor{green!30}62.8 & \cellcolor{green!30}76.9 & \cellcolor{green!30}64.7 \\
        Monte Carlo Sliced 2-Wasserstein Distance & 9.6 & 19.5 & 17.8 & 24.4 & 27.1 & 26.3 & 24 \\
        Max-Sliced 2-Wasserstein Distance & 38.7 & \cellcolor{green!30}52 & 33.3 & 26.7 & 51 & 56.8 & 54.4 \\
        Wavelet Earth Mover's Distance & 38.7 & 33.8 & 26.5 & 26 & 34.5 & 47.9 & 28.8 \\
        Sinkhorn Distance & 58.7 & 47.5 & 38.4 & 32.4 & 52.7 & 62.4 & 59.7 \\
        2-Wasserstein Distance & 53.1 & 47.3 & 42.4 & 31.5 & 44.6 & 61.9 & 58 \\
         \hline
        
        \multicolumn{8}{|c|}{\textbf{shift = 4 pixels}} \\ \hline
        ~ & \textbf{2} & \textbf{3} & \textbf{4} & \textbf{5} & \textbf{6} & \textbf{7} & \textbf{9} \\ \hline
        Euclidean Distance & 8.6 & 3.2 & 8.3 & 4.3 & 9.8 & 13.6 & 6.3 \\
        Sliced 2-Wasserstein Distance & 30.5 & 32.2 & 20.4 & 25.4 & 20.5 & 34 & 33.8 \\
        Ramp-Filtered Sliced 2-Wasserstein Distance & \cellcolor{green!30}50.8 & 28.8 & 31 & \cellcolor{green!30}40.4 & 31.6 & \cellcolor{green!30}67.6 & 36.4 \\
        Monte Carlo Sliced 2-Wasserstein Distance & 8.7 & 18.5 & 15.1 & 21.9 & 17.1 & 20.2 & 19.2 \\
        Max-Sliced 2-Wasserstein Distance & 36.6 & \cellcolor{green!30}50 & 29.9 & 26.7 & \cellcolor{green!30}50.5 & 56.8 & \cellcolor{green!30}51.2 \\
        Wavelet Earth Mover's Distance & 5 & 6.4 & 7.7 & 4.4 & 8 & 13.8 & 4.2 \\
        Sinkhorn Distance & 39.4 & 31.5 & 26.6 & 24 & 34.9 & 48.9 & 49.1 \\
        2-Wasserstein Distance & 43 & 38.9 & \cellcolor{green!30}31.1 & 28.5 & 30.7 & 55.5 & 50.6 \\
         \hline
        
        \multicolumn{8}{|c|}{\textbf{shift = 6 pixels}} \\ \hline
        ~ & \textbf{2} & \textbf{3} & \textbf{4} & \textbf{5} & \textbf{6} & \textbf{7} & \textbf{9} \\ \hline
        Euclidean Distance & 2 & 10.1 & 16.7 & 3.3 & 3.2 & 4.8 & 9.1 \\
        Sliced 2-Wasserstein Distance & 25.4 & 27.2 & 17.3 & 25 & 17.3 & 27.4 & 31.8 \\
        Ramp-Filtered Sliced 2-Wasserstein Distance & 32 & 22.3 & 21.7 & \cellcolor{green!30}35.6 & 22 & 46.5 & 27.9 \\
        Monte Carlo Sliced 2-Wasserstein Distance & 9.2 & 16.2 & 14.8 & 20.5 & 14.6 & 16.4 & 17.2 \\
        Max-Sliced 2-Wasserstein Distance & \cellcolor{green!30}36.8 & \cellcolor{green!30}49.4 & \cellcolor{green!30}27 & 28.8 & \cellcolor{green!30}49.4 & \cellcolor{green!30}56.5 & \cellcolor{green!30}48.8 \\
        Wavelet Earth Mover's Distance & 1.6 & 10.2 & 12.7 & 7.5 & 1.7 & 4.7 & 8.5 \\
        Sinkhorn Distance & 24.4 & 22 & 21.7 & 17.5 & 22.6 & 30.1 & 26.3 \\
        2-Wasserstein Distance & 34.1 & 36.7 & 26.8 & 28.2 & 24.9 & 41.9 & 40.9 \\ \hline
    \end{tabular}
    \label{tab:summary}
\end{table}

\begin{table}[!ht]
    \centering
    \small
    \caption{Timing results for the rotational alignment of the MNIST dataset for the digit ``2'' using different metrics over $3$ trials. Images are size $39 \times 39$ pixels and the number of images to be aligned is $N = 1031$. The timings were carried out on a computer with a 2.6 GHz Intel Skylake processor and 32 GB of memory. The Sinkhorn distance was computed with a regularization term $\lambda = 0.01$ and $H = 3$ iterations.}
    \begin{tabular}{lll}
        \hline
        \textbf{Metric} & \textbf{Complexity} & \textbf{Time (seconds)} \\ 
        \hline
        Euclidean Distance & $\mathcal{O}(L^2 \log L)$ & $ 0.355 \pm 0.003 $  \\ 
        Sliced 2-Wasserstein Distance & $\mathcal{O}(L^2 \log L)$ & $ 0.586 \pm 0.006 $ \\ 
        Ramp-Filtered Sliced 2-Wasserstein Distance & $\mathcal{O}(L^2 \log L)$ & $ 0.831  \pm 0.009 $ \\
        Monte Carlo Sliced 2-Wasserstein Distance & $\mathcal{O}(L^2 \log L)$ & $ 0.568 \pm 0.002 $ \\ 
        Max-Sliced 2-Wasserstein Distance & $\mathcal{O}(L^3)$ & $ 2.181 \pm 0.083 $ \\
        Wavelet Earth Mover's Distance & $\mathcal{O}(L^3 \log L)$ & $ 19.278 \pm 0.344 $ \\
        Sinkhorn Distance & $\mathcal{O}(H L^5)$ & $ 1575.567 \pm 15.473 $ \\
        2-Wasserstein Distance & $\mathcal{O}(L^7)$ & $ 1038.765 \pm 26.669 $ \\ \hline
    \end{tabular}
    \label{tab:all_timing}
\end{table}



\section*{Acknowledgments}
We thank Marc Aurèle Gilles and Amit Moscovich for valuable discussions and insight. A.S. and E.J.V. are supported in part by AFOSR FA9550-23-1-0249, the Simons Foundation Math+X Investigator Award, NSF DMS 2009753, NSF DMS 2510039, and NIH/NIGMS R01GM136780-01. Y.S. is supported by NSF DMS 2514152.

\clearpage

\printbibliography

@book{Natterer2001,
  author    = {Frank Natterer},
  title     = {The Mathematics of Computerized Tomography},
  year      = {2001},
  publisher = {SIAM},
  series    = {Classics in Applied Mathematics},
  address   = {Philadelphia}
}

@book{Evans2010,
  author    = {Lawrence C. Evans},
  title     = {Partial Differential Equations},
  edition   = {2},
  year      = {2010},
  publisher = {American Mathematical Society},
  series    = {Graduate Studies in Mathematics},
  address   = {Providence, RI}
}

@book{Zygmund2002,
  author    = {Antoni Zygmund},
  title     = {Trigonometric Series},
  edition   = {3},
  year      = {2002},
  publisher = {Cambridge University Press},
  address   = {Cambridge}
}

@article{BarnettMagland2019,
  author  = {Alex H. Barnett and Jeffrey Magland and Ludvig af Klinteberg},
  title   = {A Parallel Nonuniform Fast Fourier Transform Library Based on an Exponential of Semicircle Kernel},
  journal = {SIAM Journal on Scientific Computing},
  year    = {2019},
  volume  = {41},
  number  = {5},
  pages   = {C479--C504}
}

@book{Higham2002,
  author    = {Nicholas J. Higham},
  title     = {Accuracy and Stability of Numerical Algorithms},
  edition   = {2},
  year      = {2002},
  publisher = {SIAM},
  address   = {Philadelphia}
}

@article{contrast,
  title = {Ab-initio Contrast Estimation and Denoising of Cryo-EM Images},
  author = {Shi, Yunpeng and Singer, Amit},
  journal = {Computer Methods and Programs in Biomedicine},
  year = {2022},
  volume = {224},
  pages = {107018},
  doi = {10.1016/j.cmpb.2022.107018},
  url = {https://doi.org/10.1016/j.cmpb.2022.107018}
}

@article{fastPCA,
  title = {Fast Principal Component Analysis for Cryo-EM Images},
  author = {Marshall, Nicholas F. and Mickelin, Oscar and Shi, Yunpeng and Singer, Amit},
  journal = {Biological Imaging},
  year = {2023},
  volume = {3},
  pages = {e2},
  doi = {10.1017/S2633903X23000028},
  url = {https://doi.org/10.1017/S2633903X23000028},
}

@article{gong_radon_2023,
	title = {The {Radon} {Signed} {Cumulative} {Distribution} {Transform} and its applications in classification of {Signed} {Images}},
	url = {http://arxiv.org/abs/2307.15339},
	publisher = {arXiv},
	author = {Gong, Le and Li, Shiying and Pathan, Naqib Sad and Shifat-E-Rabbi, Mohammad and Rohde, Gustavo K. and Rubaiyat, Abu Hasnat Mohammad and Thareja, Sumati},
	month = jul,
	year = {2023}
}

@article{kolouri_radon_2016,
	title = {The {Radon} {Cumulative} {Distribution} {Transform} and {Its} {Application} to {Image} {Classification}},
	volume = {25},
	issn = {1057-7149, 1941-0042},
	url = {http://ieeexplore.ieee.org/document/7358128/},
	doi = {10.1109/TIP.2015.2509419},
	number = {2},
	journal = {IEEE Transactions on Image Processing},
	author = {Kolouri, Soheil and Park, Se Rim and Rohde, Gustavo K.},
	month = feb,
	year = {2016},
	pages = {920--934}
}

@article{park_geometry_2023,
	title = {Geometry and analytic properties of the sliced {Wasserstein} space},
	url = {http://arxiv.org/abs/2311.05134},
	publisher = {arXiv},
	author = {Park, Sangmin and Slepčev, Dejan},
	month = dec,
	year = {2023}
}

@misc{leeb_metrics_2023,
	title = {On metrics robust to noise, perturbations, and changes in projection angle},
	url = {http://arxiv.org/abs/2101.10867},
	publisher = {arXiv},
	author = {Leeb, William},
	month = aug,
	year = {2023}
}

@article{rangan_factorization_2020,
	title = {Factorization of the translation kernel for fast rigid image alignment},
	volume = {36},
	issn = {0266-5611, 1361-6420},
	url = {https://iopscience.iop.org/article/10.1088/1361-6420/ab4e66},
	doi = {10.1088/1361-6420/ab4e66},
	number = {2},
	journal = {Inverse Problems},
	author = {Rangan, Aaditya and Spivak, Marina and Andén, Joakim and Barnett, Alex},
	month = feb,
	year = {2020},
	pages = {024001}
}

@article{rangan_radial_2023,
	title = {Radial recombination for rigid rotational alignment of images and volumes},
	volume = {39},
	issn = {0266-5611, 1361-6420},
	url = {https://iopscience.iop.org/article/10.1088/1361-6420/aca047},
	doi = {10.1088/1361-6420/aca047},
	number = {1},
	journal = {Inverse Problems},
	author = {Rangan, Aaditya V},
	month = jan,
	year = {2023},
	pages = {015003}
}

@article{bonneel_sliced_2015,
	title = {Sliced and {Radon} {Wasserstein} {Barycenters} of {Measures}},
	volume = {51},
	issn = {0924-9907, 1573-7683},
	url = {http://link.springer.com/10.1007/s10851-014-0506-3},
	doi = {10.1007/s10851-014-0506-3},
	number = {1},
	journal = {Journal of Mathematical Imaging and Vision},
	author = {Bonneel, Nicolas and Rabin, Julien and Peyré, Gabriel and Pfister, Hanspeter},
	month = jan,
	year = {2015},
	pages = {22--45}
}

@article{shifat-e-rabbi_invariance_2023,
	title = {Invariance encoding in sliced-{Wasserstein} space for image classification with limited training data},
	volume = {137},
	issn = {00313203},
	url = {https://linkinghub.elsevier.com/retrieve/pii/S0031320322007476},
	doi = {10.1016/j.patcog.2022.109268},
	language = {en},
	journal = {Pattern Recognition},
	author = {Shifat-E-Rabbi, Mohammad and Zhuang, Yan and Li, Shiying and Rubaiyat, Abu Hasnat Mohammad and Yin, Xuwang and Rohde, Gustavo K.},
	month = may,
	year = {2023},
	pages = {109268}
}

@article{rao_wasserstein_2020,
	title = {Wasserstein {K}-{Means} for {Clustering} {Tomographic} {Projections}},
	url = {http://arxiv.org/abs/2010.09989},
	publisher = {arXiv},
	author = {Rao, Rohan and Moscovich, Amit and Singer, Amit},
	month = oct,
	year = {2020}
}

@article{singer_alignment_2024, 
        title={Alignment of density maps in Wasserstein distance}, 
        volume={4},
        DOI={10.1017/S2633903X24000059}, 
        journal={Biological Imaging}, 
        author={Singer, Amit and Yang, Ruiyi}, 
        year={2024},
        pages={e5}
}

@article{riahi_empot_2023,
	title = {{EMPOT}: partial alignment of density maps and rigid body fitting using unbalanced {Gromov}-{Wasserstein} divergence},
	url = {http://arxiv.org/abs/2311.00850},
	publisher = {arXiv},
	author = {Riahi, Aryan Tajmir and Zhang, Chenwei and Chen, James and Condon, Anne and Duc, Khanh Dao},
	month = nov,
	year = {2023}
}

@article{riahi_alignot_2023,
	title = {{AlignOT}: {An} {Optimal} {Transport} {Based} {Algorithm} for {Fast} {3D} {Alignment} {With} {Applications} to {Cryogenic} {Electron} {Microscopy} {Density} {Maps}},
	volume = {20},
	issn = {1545-5963, 1557-9964, 2374-0043},
	url = {https://ieeexplore.ieee.org/document/10298813/},
	doi = {10.1109/TCBB.2023.3327633},
	number = {6},
	journal = {IEEE/ACM Transactions on Computational Biology and Bioinformatics},
	author = {Riahi, Aryan Tajmir and Woollard, Geoffrey and Poitevin, Frédéric and Condon, Anne and Duc, Khanh Dao},
	month = nov,
	year = {2023},
	pages = {3842--3850}
}

@inproceedings{zelesko_earthmover-based_2020,
	title = {Earthmover-{Based} {Manifold} {Learning} for {Analyzing} {Molecular} {Conformation} {Spaces}},
	isbn = {978-1-5386-9330-8},
	url = {https://ieeexplore.ieee.org/document/9098723/},
	doi = {10.1109/ISBI45749.2020.9098723},
	booktitle = {2020 {IEEE} 17th {International} {Symposium} on {Biomedical} {Imaging} ({ISBI})},
	publisher = {IEEE},
	author = {Zelesko, Nathan and Moscovich, Amit and Kileel, Joe and Singer, Amit},
	month = apr,
	year = {2020},
	pages = {1715--1719}
}

@article{rabin_wasserstein_2012,
	address = {Berlin, Heidelberg},
	title = {Wasserstein {Barycenter} and {Its} {Application} to {Texture} {Mixing}},
	volume = {6667},
	isbn = {978-3-642-24784-2},
	url = {http://link.springer.com/10.1007/978-3-642-24785-9_37},
	booktitle = {Scale {Space} and {Variational} {Methods} in {Computer} {Vision}},
	publisher = {Springer Berlin Heidelberg},
	author = {Rabin, Julien and Peyré, Gabriel and Delon, Julie and Bernot, Marc},
	year = {2012},
	doi = {10.1007/978-3-642-24785-9_37},
	pages = {435--446},
}

@book{villani_optimal_2009,
	title = {Optimal {Transport}},
	volume = {338},
	copyright = {http://www.springer.com/tdm},
	isbn = {978-3-540-71050-9},
	url = {http://link.springer.com/10.1007/978-3-540-71050-9},
	publisher = {Springer Berlin Heidelberg},
	author = {Villani, Cédric},
	year = {2009},
	doi = {10.1007/978-3-540-71050-9},
}

@book{santambrogio_optimal_2015,
	title = {Optimal {Transport} for {Applied} {Mathematicians}: {Calculus} of {Variations}, {PDEs}, and {Modeling}},
	volume = {87},
	copyright = {https://www.springernature.com/gp/researchers/text-and-data-mining},
	isbn = {978-3-319-20827-5},
	url = {https://link.springer.com/10.1007/978-3-319-20828-2},
	publisher = {Springer International Publishing},
	author = {Santambrogio, Filippo},
	year = {2015},
	doi = {10.1007/978-3-319-20828-5},
}

@misc{peyre_computational_2020,
	title = {Computational {Optimal} {Transport}},
	url = {http://arxiv.org/abs/1803.00567},
	publisher = {arXiv},
	author = {Peyré, Gabriel and Cuturi, Marco},
	month = mar,
	year = {2020},
}

@inproceedings{zehni_cryoswd_2023,
	title = {{CryoSWD}: {Sliced} {Wasserstein} {Distance} {Minimization} for {3D} {Reconstruction} in {Cryo}-electron {Microscopy}},
	copyright = {https://doi.org/10.15223/policy-029},
	isbn = {978-1-72816-327-7},
	url = {https://ieeexplore.ieee.org/document/10095016/},
	doi = {10.1109/ICASSP49357.2023.10095016},
	booktitle = {{ICASSP} 2023 - 2023 {IEEE} {International} {Conference} on {Acoustics}, {Speech} and {Signal} {Processing} ({ICASSP})},
	publisher = {IEEE},
	author = {Zehni, Mona and Zhao, Zhizhen},
	month = jun,
	year = {2023},
	pages = {1--5},
}

@phdthesis{bonnotte2013unidimensional,
  title={Unidimensional and evolution methods for optimal transportation},
  author={Bonnotte, Nicolas},
  year={2013},
  school={Universit{\'e} Paris Sud-Paris XI; Scuola normale superiore (Pise, Italie)}
}

@inproceedings{shirdhonkar_approximate_2008,
	address = {Anchorage, AK, USA},
	title = {Approximate earth mover's distance in linear time},
	isbn = {978-1-4244-2242-5},
	url = {http://ieeexplore.ieee.org/document/4587662/},
	doi = {10.1109/CVPR.2008.4587662},
	booktitle = {2008 {IEEE} {Conference} on {Computer} {Vision} and {Pattern} {Recognition}},
	publisher = {IEEE},
	author = {Shirdhonkar, Sameer and Jacobs, David W.},
	month = jun,
	year = {2008},
	pages = {1--8},
}

@article{heimowitz_centering_2021,
	title = {Centering {Noisy} {Images} with {Application} to {Cryo}-{EM}},
	volume = {14},
	issn = {1936-4954},
	url = {https://epubs.siam.org/doi/10.1137/20M1365946},
	doi = {10.1137/20M1365946},
	number = {2},
	journal = {SIAM Journal on Imaging Sciences},
	author = {Heimowitz, Ayelet and Sharon, Nir and Singer, Amit},
	month = jan,
	year = {2021},
	pages = {689--716},
}

@article{sorzano_bias_2022,
	title = {On bias, variance, overfitting, gold standard and consensus in single-particle analysis by cryo-electron microscopy},
	volume = {78},
	issn = {2059-7983},
	url = {https://scripts.iucr.org/cgi-bin/paper?S2059798322001978},
	doi = {10.1107/S2059798322001978},
	number = {4},
	journal = {Acta Crystallographica Section D Structural Biology},
	author = {Sorzano, C. O. S. and Jiménez-Moreno, A. and Maluenda, D. and Martínez, M. and Ramírez-Aportela, E. and Krieger, J. and Melero, R. and Cuervo, A. and Conesa, J. and Filipovic, J. and Conesa, P. and Del Caño, L. and Fonseca, Y. C. and Jiménez-de La Morena, J. and Losana, P. and Sánchez-García, R. and Strelak, D. and Fernández-Giménez, E. and De Isidro-Gómez, F. P. and Herreros, D. and Vilas, J. L. and Marabini, R. and Carazo, J. M.},
	month = apr,
	year = {2022},
	pages = {410--423},
}

@inproceedings{cuturi_sinkhorn_2013,
 author = {Cuturi, Marco},
 booktitle = {Advances in Neural Information Processing Systems},
 editor = {C.J. Burges and L. Bottou and M. Welling and Z. Ghahramani and K.Q. Weinberger},
 publisher = {Curran Associates, Inc.},
 title = {Sinkhorn Distances: Lightspeed Computation of Optimal Transport},
 url = {https://proceedings.neurips.cc/paper_files/paper/2013/file/af21d0c97db2e27e13572cbf59eb343d-Paper.pdf},
 volume = {26},
 year = {2013}
}

@article{avants_symmetric_2008,
	title = {Symmetric diffeomorphic image registration with cross-correlation: {Evaluating} automated labeling of elderly and neurodegenerative brain},
	volume = {12},
	copyright = {https://www.elsevier.com/tdm/userlicense/1.0/},
	issn = {13618415},
	shorttitle = {Symmetric diffeomorphic image registration with cross-correlation},
	url = {https://linkinghub.elsevier.com/retrieve/pii/S1361841507000606},
	doi = {10.1016/j.media.2007.06.004},
	number = {1},
	journal = {Medical Image Analysis},
	author = {Avants, B and Epstein, C and Grossman, M and Gee, J},
	month = feb,
	year = {2008},
	pages = {26--41},
}

@article{toader_methods_2023,
	title = {Methods for {Cryo}-{EM} {Single} {Particle} {Reconstruction} of {Macromolecules} {Having} {Continuous} {Heterogeneity}},
	volume = {435},
	issn = {00222836},
	url = {https://linkinghub.elsevier.com/retrieve/pii/S0022283623000761},
	doi = {10.1016/j.jmb.2023.168020},
	number = {9},
	journal = {Journal of Molecular Biology},
	author = {Toader, Bogdan and Sigworth, Fred J. and Lederman, Roy R.},
	month = may,
	year = {2023},
	pages = {168020},
}

@article{lintott_galaxy_2008,
	title = {Galaxy {Zoo}: morphologies derived from visual inspection of galaxies from the {Sloan} {Digital} {Sky} {Survey}},
	volume = {389},
	issn = {00358711, 13652966},
	url = {https://academic.oup.com/mnras/article-lookup/doi/10.1111/j.1365-2966.2008.13689.x},
	doi = {10.1111/j.1365-2966.2008.13689.x},
	number = {3},
	journal = {Monthly Notices of the Royal Astronomical Society},
	author = {Lintott, Chris J. and Schawinski, Kevin and Slosar, Anže and Land, Kate and Bamford, Steven and Thomas, Daniel and Raddick, M. Jordan and Nichol, Robert C. and Szalay, Alex and Andreescu, Dan and Murray, Phil and Vandenberg, Jan},
	month = sep,
	year = {2008},
	pages = {1179--1189},
}

@article{singer_computational_2020,
	title = {Computational {Methods} for {Single}-{Particle} {Electron} {Cryomicroscopy}},
	volume = {3},
	issn = {2574-3414, 2574-3414},
	url = {https://www.annualreviews.org/doi/10.1146/annurev-biodatasci-021020-093826},
	doi = {10.1146/annurev-biodatasci-021020-093826},
	number = {1},
	journal = {Annual Review of Biomedical Data Science},
	author = {Singer, Amit and Sigworth, Fred J.},
	month = jul,
	year = {2020},
	pages = {163--190},
}

@article{tang_conformational_2023,
	title = {Conformational heterogeneity and probability distributions from single-particle cryo-electron microscopy},
	volume = {81},
	issn = {0959440X},
	url = {https://linkinghub.elsevier.com/retrieve/pii/S0959440X23001008},
	doi = {10.1016/j.sbi.2023.102626},
	journal = {Current Opinion in Structural Biology},
	author = {Tang, Wai Shing and Zhong, Ellen D. and Hanson, Sonya M. and Thiede, Erik H. and Cossio, Pilar},
	month = aug,
	year = {2023},
	pages = {102626},
}

@article{kolouri_optimal_2017,
  author={Kolouri, Soheil and Park, Se Rim and Thorpe, Matthew and Slepcev, Dejan and Rohde, Gustavo K.},
  journal={IEEE Signal Processing Magazine}, 
  title={Optimal Mass Transport: Signal processing and machine-learning applications}, 
  year={2017},
  volume={34},
  number={4},
  pages={43-59},
  doi={10.1109/MSP.2017.2695801}
}

@article{dutt_fast_1993,
	title = {Fast {Fourier} {Transforms} for {Nonequispaced} {Data}},
	volume = {14},
	issn = {1064-8275, 1095-7197},
	url = {http://epubs.siam.org/doi/10.1137/0914081},
	doi = {10.1137/0914081},
	number = {6},
	journal = {SIAM Journal on Scientific Computing},
	author = {Dutt, A. and Rokhlin, V.},
	month = nov,
	year = {1993},
	pages = {1368--1393},
}

@article{barnett_parallel_2019,
	title = {A {Parallel} {Nonuniform} {Fast} {Fourier} {Transform} {Library} {Based} on an “{Exponential} of {Semicircle}" {Kernel}},
	volume = {41},
	issn = {1064-8275, 1095-7197},
	url = {https://epubs.siam.org/doi/10.1137/18M120885X},
	doi = {10.1137/18M120885X},
	number = {5},
	journal = {SIAM Journal on Scientific Computing},
	author = {Barnett, Alexander H. and Magland, Jeremy and Af Klinteberg, Ludvig},
	month = jan,
	year = {2019},
	pages = {C479--C504},
}

@article{yang_cryo-em_2008,
	title = {Cryo-{EM} image alignment based on nonuniform fast {Fourier} transform},
	volume = {108},
	copyright = {https://www.elsevier.com/tdm/userlicense/1.0/},
	issn = {03043991},
	url = {https://linkinghub.elsevier.com/retrieve/pii/S0304399108000533},
	doi = {10.1016/j.ultramic.2008.03.006},
	number = {9},
	journal = {Ultramicroscopy},
	author = {Yang, Zhengfan and Penczek, Pawel A.},
	month = aug,
	year = {2008},
	pages = {959--969},
}

@article{rubner_earth_2000,
	title = {The {Earth} {Mover}'s {Distance} as a {Metric} for {Image} {Retrieval}},
	volume = {40},
	issn = {1573-1405},
	url = {https://doi.org/10.1023/A:1026543900054},
	doi = {10.1023/A:1026543900054},
	number = {2},
	journal = {International Journal of Computer Vision},
	author = {Rubner, Yossi and Tomasi, Carlo and Guibas, Leonidas J.},
	month = nov,
	year = {2000},
	pages = {99--121},
}

@INPROCEEDINGS{kolouri_generalized_2022,
  author={Kolouri, Soheil and Nadjahi, Kimia and Shahrampour, Shahin and Şimşekli, Umut},
  booktitle={ICASSP 2022 - 2022 IEEE International Conference on Acoustics, Speech and Signal Processing (ICASSP)}, 
  title={Generalized Sliced Probability Metrics}, 
  year={2022},
  pages={4513-4517},
  doi={10.1109/ICASSP43922.2022.9746016}
}

@INPROCEEDINGS{deshpande_max_2019,
  author={Deshpande, Ishan and Hu, Yuan-Ting and Sun, Ruoyu and Pyrros, Ayis and Siddiqui, Nasir and Koyejo, Sanmi and Zhao, Zhizhen and Forsyth, David and Schwing, Alexander G.},
  booktitle={2019 IEEE/CVF Conference on Computer Vision and Pattern Recognition (CVPR)}, 
  title={Max-Sliced Wasserstein Distance and Its Use for GANs}, 
  year={2019},
  pages={10640-10648},
  doi={10.1109/CVPR.2019.01090}
}

@ARTICLE{reddy_fft_1996,
  author={Reddy, B.S. and Chatterji, B.N.},
  journal={IEEE Transactions on Image Processing}, 
  title={An FFT-based technique for translation, rotation, and scale-invariant image registration}, 
  year={1996},
  volume={5},
  number={8},
  pages={1266-1271},
  doi={10.1109/83.506761}
}

@article{solomon_convolutional_2015,
author = {Solomon, Justin and de Goes, Fernando and Peyr\'{e}, Gabriel and Cuturi, Marco and Butscher, Adrian and Nguyen, Andy and Du, Tao and Guibas, Leonidas},
title = {Convolutional wasserstein distances: efficient optimal transportation on geometric domains},
year = {2015},
issue_date = {August 2015},
publisher = {Association for Computing Machinery},
address = {New York, NY, USA},
volume = {34},
number = {4},
issn = {0730-0301},
doi = {10.1145/2766963},
journal = {ACM Trans. Graph.},
month = jul,
articleno = {66}
}

@ARTICLE{bonneel_survey_2023,
  title     = "A survey of optimal transport for computer graphics and computer
               vision",
  author    = "Bonneel, Nicolas and Digne, Julie",
  journal   = "Comput. Graph. Forum",
  publisher = "Wiley",
  volume    =  42,
  number    =  2,
  pages     = "439--460",
  month     =  may,
  year      =  2023,
}

@inproceedings{kolouri_generalized_2019,
 author = {Kolouri, Soheil and Nadjahi, Kimia and Simsekli, Umut and Badeau, Roland and Rohde, Gustavo},
 booktitle = {Advances in Neural Information Processing Systems},
 editor = {H. Wallach and H. Larochelle and A. Beygelzimer and F. d\textquotesingle Alch\'{e}-Buc and E. Fox and R. Garnett},
 pages = {},
 publisher = {Curran Associates, Inc.},
 title = {Generalized Sliced Wasserstein Distances},
 volume = {32},
 year = {2019}
}

@InProceedings{Kolouri_2016_CVPR,
author = {Kolouri, Soheil and Zou, Yang and Rohde, Gustavo K.},
title = {Sliced Wasserstein Kernels for Probability Distributions},
booktitle = {Proceedings of the IEEE Conference on Computer Vision and Pattern Recognition (CVPR)},
year = {2016}
}

@misc{Deshpande_2018_generative,
      title={Generative Modeling using the Sliced Wasserstein Distance}, 
      author={Ishan Deshpande and Ziyu Zhang and Alexander Schwing},
      year={2018},
      archivePrefix={arXiv},
      primaryClass={cs.CV},
      url={https://arxiv.org/abs/1803.11188}, 
}

@article{flamary2021pot,
  author  = {R{\'e}mi Flamary and Nicolas Courty and Alexandre Gramfort and Mokhtar Z. Alaya and Aur{\'e}lie Boisbunon and Stanislas Chambon and Laetitia Chapel and Adrien Corenflos and Kilian Fatras and Nemo Fournier and L{\'e}o Gautheron and Nathalie T.H. Gayraud and Hicham Janati and Alain Rakotomamonjy and Ievgen Redko and Antoine Rolet and Antony Schutz and Vivien Seguy and Danica J. Sutherland and Romain Tavenard and Alexander Tong and Titouan Vayer},
  title   = {POT: Python Optimal Transport},
  journal = {Journal of Machine Learning Research},
  year    = {2021},
  volume  = {22},
  number  = {78},
  pages   = {1-8},
  url     = {http://jmlr.org/papers/v22/20-451.html}
}

@misc{sisouk_users_2025,
	title = {A {User}'s {Guide} to {Sampling} {Strategies} for {Sliced} {Optimal} {Transport}},
	copyright = {arXiv.org perpetual, non-exclusive license},
	url = {https://arxiv.org/abs/2502.02275},
	doi = {10.48550/ARXIV.2502.02275},
	publisher = {arXiv},
	author = {Sisouk, Keanu and Delon, Julie and Tierny, Julien},
	year = {2025},
}

@article{lakshmanan_nonequispaced_2023,
	title = {Nonequispaced {Fast} {Fourier} {Transform} {Boost} for the {Sinkhorn} {Algorithm}},
	volume = {58},
	issn = {1068-9613, 1068-9613},
	url = {https://hw.oeaw.ac.at?arp=0x003e223f},
	doi = {10.1553/etna_vol58s289},
	journal = {ETNA - Electronic Transactions on Numerical Analysis},
	author = {Lakshmanan, Rajmadan and Pichler, Alois and Potts, Daniel},
	year = {2023},
	pages = {289--315},
}

@article{cramer_composition_1928,
	title = {On the composition of elementary errors: {First} paper: {Mathematical} deductions},
	volume = {1928},
	issn = {0346-1238, 1651-2030},
	shorttitle = {On the composition of elementary errors},
	url = {http://www.tandfonline.com/doi/abs/10.1080/03461238.1928.10416862},
	doi = {10.1080/03461238.1928.10416862},
	number = {1},
	journal = {Scandinavian Actuarial Journal},
	author = {Cramér, Harald},
	month = jan,
	year = {1928},
	pages = {13--74},
}

\end{document}